\documentclass[english]{article}
\usepackage[T1]{fontenc}
\usepackage[latin9]{inputenc}
\usepackage{babel}
\usepackage{hyperref}
\usepackage{amssymb}
\usepackage{amsmath}
\usepackage{amsthm}
\usepackage{mathtools}
\usepackage{amsfonts}
\usepackage[extdef=true]{delimset}
\usepackage{dsfont}
\usepackage{xcolor}
\usepackage{scalerel}
\usepackage{algorithm}
\usepackage[noend]{algorithmic}
\usepackage{natbib}
\setcitestyle{authoryear,round,citesep={;},aysep={,},yysep={;}}

\usepackage[capitalize]{cleveref}
\usepackage{makecell}
\usepackage{threeparttable}
\usepackage{microtype}
\usepackage{graphicx}
\usepackage{subfigure}
\usepackage{booktabs} % for professional tables
\usepackage{geometry}

\newtheorem{theorem}{Theorem}
\newtheorem{lemma}{Lemma}

\newtheorem{corollary}{Corollary}

\newtheorem{remark}{Remark}

\crefname{claim}{claim}{claims}
\Crefname{algorithm}{Algorithm}{Algorithms}
\Crefname{theorem}{Theorem}{Theorems}
\crefname{protocol}{protocol}{protocols}
\newcommand{\E}{\mathbb{E}}
\newcommand{\N}{\mathbb{N}}
\newcommand{\dmax}{d_{\scaleto{\text{max}}{3pt}}}

\DeclareMathOperator*{\argmax}{arg\,max}

\newcommand{\indicator}{\mathds{1}}

\newcommand{\Deltabar}{\overline{\Delta}}
\newcommand{\Ical}{\mathcal{I}}
\newcommand{\Iber}{\Ical_{ber}}
\newcommand{\Rcal}{\mathcal{R}}
\newcommand{\mabregret}{\mathcal{R}^{\scaleto{MAB}{3pt}}_T}

\newcommand{\ifrac}[2]{#1/#2}

\author{
Tal Lancewicki$^{1\dag}$
\and
Shahar Segal$^{1\dag}$
\and
Tomer Koren$^{1,2}$
\and
Yishay Mansour$^{1,2}$
}

\begin{document}
\title{Stochastic Multi-Armed Bandits with Unrestricted Delay Distributions}
\maketitle

\def\thefootnote{$\dag$ }\footnotetext{ These authors contributed equally to this work.}
\def\thefootnote{1 }\footnotetext{ Blavatnik School of Computer Science, Tel Aviv University, Israel.}
\def\thefootnote{2 }\footnotetext{ Google Research, Tel Aviv.}

\begin{abstract}
    We study the stochastic Multi-Armed Bandit~(MAB) problem with 
    random delays in the feedback received by the algorithm. We consider two settings: the {\it reward-dependent} delay setting, where realized delays may depend on the stochastic rewards, and the {\it reward-independent} delay setting. Our main contribution is algorithms that achieve near-optimal regret in each of the settings, with an additional additive dependence on the quantiles of the delay distribution.
    Our results do not make any assumptions on the delay distributions: in particular, we do not assume they come from any parametric family of distributions and allow for unbounded support and expectation; 
    we further allow for infinite delays where the algorithm might occasionally not observe any feedback.
    % we further allow for the case of infinite delays where occasionally no feedback is observed.
\end{abstract}

\section{Introduction}
Stochastic Multi-armed Bandit problem (MAB) is a theoretical framework for studying sequential decision making. 
Most of the literature on MAB assumes that the agent observes feedback immediately after taking an action. However, in many real world applications, 
%, such as medical trails and recommendation systems
the feedback might be available only after a period of time. For instance, in clinical trials, the observed effect of a medical treatment often comes in delay, that may vary between different treatments.
% Another example comes for the world of recommendation systems, where a user conversion might occur some time after a sent recommendation. Similarly, the delay varies between different recommendations and the conversion itself. In order to model these cases, it is important to implement a delayed feedback environment and consider how the delay affects the agent. 
% \tal{We're not modeling delay that depends on the patients/user so I don't we think we should mention it.}
% \shahar{Right, I'll change that.}
Another example is in targeted advertising on the web: when a user clicks a display ad the feedback is immediate, but if a user decides \emph{not} to click, then the algorithm will become aware to that only when the user left the website or enough time has elapsed.

In this paper, we study the stochastic MAB problem with randomized delays~\citep{joulani2013online}. The reward of the chosen action at time $t$ is sampled from some distribution, like in the classic stochastic MAB problem. However, the reward is \textit{observed} only at time $t+d_t$, where $d_t$ is a random variable denoting the delay at step $t$. 
This problem has been studied extensively in the literature~\citep{joulani2013online,vernade2017stochastic,pike2018bandits,manegueu2020stochastic} under an implicit assumption that the delays are \emph{reward-independent}: namely, that $d_t$ is sampled from an unknown delay distribution and may depend on the chosen arm, but \emph{not} on the stochastic rewards on the same round.
For example, \citet{joulani2013online,pike2018bandits} show a regret bound of the form $O(\mabregret + K\mathbb{E}[D])$. Here $\mabregret$ denotes the optimal instance-dependent $T$-round regret bound for standard (non-delayed) MAB: $\mabregret = \sum_{\Delta_i > 0}\log(T)/\Delta_i$, where $\Delta_i$ is the sub-optimality gap for arm $i$. In the second term, $K$ is the number of arms and $\mathbb{E}[D]$ is the expected delay. 

A significantly more challenging setting, that to the best of our knowledge was not explicitly addressed previously in the literature,%
\footnote{Some of the results of \citet{vernade2017stochastic,manegueu2020stochastic} can be viewed as having a specific form of reward-dependent delays; we discuss this in more detail in the related work section.} 
is that of \emph{reward-dependent} delays. In this setting, the random delay at each round may also depend on the reward received on the same round (in other words, they are drawn together from a \emph{joint distribution} over rewards and delays). This scenario is motivated by both of the examples mentioned earlier: e.g., in targeted advertisement the delay associated with a certain user is strongly correlated with the reward she generates (i.e., click or no click); and in clinical trials, the delay often depends on the effect of the applied treatment as some side-effects take longer than others to surface.

In contrast to the reward-independent case, with reward-dependent delays the observed feedback might give a biased impression of the true rewards. Namely, the expectation of the \textit{observed} reward can be very different than the actual expected reward. For example, consider Bernoulli rewards. If the delays given reward $0$ are shorter than the delays given reward $1$, then the observed reward will be biased towards~$0$. Even worse, the direction of the bias can be opposite between different arms. Hence, as long as the fraction of \textit{unobserved} feedback is significant, the expected \textit{observed} reward of the optimal arm can be smaller than expected \textit{observed} reward of a sub-optimal arm, which makes the learning task substantially more challenging.

\subsection{Our contributions}

We consider both the reward-independent and reward-dependent versions of stochastic MAB with delays.
% , and give algorithms that achieve near-optimal regret guarantees and provide an instance-dependent lower bound. 
In the reward-independent case we give new algorithms whose regret bounds significantly improve upon the state-of-the-art, and also give instance-dependent lower bounds demonstrating that our algorithms are nearly-optimal.
In the reward-dependent setting, we give the first algorithm 
to handle such delay structure and the potential bias in the observed feedback that it induces. We provide both an upper bound on the regret and a nearly matching general lower bound.

\paragraph{Reward-independent delays:}

% In the reward-\emph{independent} setting, \citet{joulani2013online,pike2018bandits} show a regret bound of the form $O(\mabregret + K\mathbb{E}[d])$. Here and throughout, $\mabregret$ denotes the optimal instance-dependent $T$-round regret bound for standard (non-delayed) MAB: $\mabregret = \sum_{\Delta_i > 0}\log(T)/\Delta_i$, where $\Delta_i$ is the sub-optimality gap for arm $i$. In the second term, $K$ is the number of arms and $\mathbb{E}[d]$ is the expected delay. 
%%%%\citet{pike2018bandits} consider a more challenging setting in which the learner observe the \textit{sum} of rewards that arrive at the same round, and show similar regret bound.  
We first consider the easier \emph{reward-independent} case.
In this case, we provide an algorithm where the second term scales with a quantile of the delay distribution rather the expectation, and the regret is bounded by $O(\min_{q}\{\mabregret/q + d(q)\})$, where $d(q)$ is the $q$-quantile of the delay distribution. Specifically, when choosing the median (i.e., $q = 1/2$), we obtain regret bound of $O(\mabregret + d(1/2))$. 
We thus improve over the $O(\mabregret + K\mathbb{E}[D])$ regret bound of \citet{joulani2013online,pike2018bandits}, as the median is always smaller than the expectation, up to factor of two (for non-negative random variables). 
Moreover, the increase in regret due to delays in our bound does not scale with number of arms, so the improvement is significant even with fixed delays~\citep{dudik2011efficient,joulani2013online}. 
Our bound is achieved using a remarkably simple algorithm, based on variant of Successive Elimination~\citep{even2006action}. 
For this algorithm, we also prove a more delicate regret bound for arm-dependent delays that allows for choosing different quantiles $q_i$ for different arms $i$ (rather than a single quantile $q$ for all arms simultaneously).
% unlike previous bounds, the increase in delay does not depend on the number  even for the . We achieve this bound using a variant of Successive Elimination algorithm \citep{even2006action}.

% Specifically, consider the median rather than the expectation. The median is always smaller than twice the expectation (for non-negative random variables), and unlike the expectation, it is always bounded for a distribution over the naturals. Thus, 

The intuition why the increase in regret due to delays should scale with a certain quantile is fairly straightforward: consider for instance the median of the delay, $d_M$. For simplicity, assume that the delay value is available when we take the action. One can simulate a black box algorithm for delays that are bounded by $d_M$ on the rounds in which delay is smaller than $d_M$ (which are approximately half of the rounds), and in the rest of the rounds, imitate the last action of the black-box algorithm. Since rewards are stochastic, and independent of time and the delay, the regret on rounds with delay larger than $d_M$ is similar to the regret of the black-box algorithm on the rest of the rounds, resulting with total regret of twice the regret of the black-box algorithm. For example, when using the algorithm of \cite{joulani2013online}, this would give us $O(\mabregret + Kd_M)$. We stress that unlike this reduction, our algorithm does not need to know the value of the delay at any time, nor the median or any other quantile. In addition, our bound is much stronger and does not depend on $K$ on the second term.

\paragraph{Reward-dependent delays:}
We then proceed to consider the more challenging \emph{reward-dependent} setting.
In this setting, the feedback reveals much less information on the true rewards due to the selection bias in the \emph{observed} rewards (in other words, the distributions of the observed feedback and the unobserved feedback might be very different).
In order to deal with this uncertainty, we present another algorithm, also inspired by Successive Elimination. The algorithm widens the confidence bounds in order to handle the potential bias. We achieve a regret bound of the form $O(\mabregret + \log(K)d(1-\Delta_{min}/4))$, where $\Delta_{min}$ is the minimal sub-optimality gap, and $d(\cdot)$ is the quantile function of the marginal delay distribution. We show that this bound is optimal, by presenting a matching lower bound, up to a factor of $\Delta$ in the second term (and $\log{(K)}$ factors).

% Our main results, together with concise comparison to previous work, in various settings, are presented in \Cref{table:main_results}. In \Cref{table:main_results},  the ``$\alpha$-Pareto'' setting assumes that the CDF of the delay is bounded from below by a CDF of an $\alpha$-Pareto distribution as in \cite{manegueu2020stochastic}. In ``Packet loss'' setting, the delay is $0$ with probability $p\in(0,1)$, and $\infty$ (or $T$) otherwise. $G^{*}_{T}$ is maximal number of unobserved feedback.

% *****************

\paragraph{Summary and comparison of bounds:}

Our main results, along with a concise comparison to previous work, are presented in~\cref{table:main_results}. $G^{*}_{T,i}$ denotes the maximal number of unobserved feedback from arm $i$.
The results show that our algorithm works well even under heavy-tailed distributions and some distributions with infinite expected value. For example, the arm-dependent delay distributions used by \citet{manegueu2020stochastic} are all bounded by an $\alpha$-pareto distribution (in terms of the delay distributions CDFs). Hence, their median is bounded by $2^{1/\alpha}$. Our algorithm suffer at most an additional $O(2^{1/\alpha})$ to the classical regret for MAB without delays (see bounds for the $\alpha$-Pareto case in~\cref{table:main_results}). 
In the ``packet loss'' setting, the delay is $0$ with probability $p$, and $\infty$ (or $T$) otherwise.
If $p$ is a constant (e.g., $>1/4$), our regret bound scales as the optimal regret bound for MAB without delays, up to constant factors. Previous work \citet{joulani2013online} show a regret bound which scales with the number of missing samples, and thus is linear. A Pareto distribution that will bound such delay would require a very small parameter $\alpha$ which also result in linear regret bound by the result of \citet{manegueu2020stochastic}. 
%
% $\mathcal{R}^{MAB}_T$ is the optimal instance-dependent bound for standard (non-delayed) MAB, over $T$ rounds: $\sum_{\Delta_i > 0}\log(T)/\Delta_i$, where $\Delta_i$ is the difference between the expected reward of arm $i$ and the optimal arm. $G^{*}_{T}$ is maximal number of outstanding feedback. $\Delta_{\min}$ is the minimal, positive $\Delta_i$. For last, $d(q)$ is the $q$-quantile of the delay distribution (in the reward-dependent setting: marginal distribution), where $q\in (0,1]$.

\begin{table}[t]
    \vskip -0.3cm
    \centering\small
    %\begin{threeparttable}[b]
        % \caption{Comparison between our regret bound and the best bound so far, 
        % %to the best of our knowledge, 
        % in various settings.
        % %
        % $\mathcal{R}^{\scaleto{MAB}{3pt}}_T$ is the optimal instance-dependent bound for standard (non-delayed) MAB over $T$ rounds: $\sum_{\Delta_i > 0}\log(T)/\Delta_i$, where $\Delta_i$ is the sub-optimality gap of arm $i$. $\Delta_{min}$ is the minimal positive $\Delta_i$, $d(q)$ is the $q$-quantile of the delay distribution and $G^{*}_{T,i}$ is the maximal number of outstanding feedback from arm $i$. 
        % % \tomer{you could use a shorter notation for $\mathcal{R}^{MAB}_T$ to make the table more condense}
        % }
        \caption{Regret bounds comparison of this and previous works. The bounds in this table omit constant and $\log(K)$ factors. %\shahar{art? short for articles?}
        }
        \setlength{\tabcolsep}{12pt} % Default value: 6pt
        \renewcommand{\arraystretch}{1} % Default value: 1
        \begin{tabular}[]{|c|c|c|}
        \hline
            & Previous work & \makecell{This paper}
            \\
        \hline
            \makecell{General, Reward-independent}
            & \makecell{
                $\mabregret +         \sum_i\mathbb{E}[G_{T,i}^{*}]$ \\
                $\mabregret + K\mathbb{E}[D]$\\
                \cite{joulani2013online}}
            & $\min_q \{\frac{1}{q}\mabregret + d(q)\}$
            \\
        \hline
            Fixed delay $d$
            & \makecell{
                $\mabregret + Kd$\\
                \cite{joulani2013online}\\
                $\sqrt{TK} + \sqrt{K}d$\\
                \cite{dudik2011efficient}}
                & $\mabregret + d$ 
                \\
        \hline
            \makecell{$\alpha$-Pareto
            %\tnote{1}
            }
            & \makecell{
                $\mabregret + \sum_i\left( \frac{8}{\Delta_i} \right)^{\frac{1-\alpha}{\alpha}}$\\
                \cite{manegueu2020stochastic}}
            & $\mabregret + 2^{1/\alpha}$
            \\
        \hline
            \makecell{Packet loss
            %\tnote{2}
            }
            & \makecell{$(1-p)T$ \\ \cite{joulani2013online}}
            & $\frac{1}{p}\mabregret$ 
            \\
        \hline
            \makecell{General, Reward-dependent
            %\tnote{3}
            }
            & \textemdash
            & \makecell{$\mabregret$ $+ d(1-\Delta_{min})$}
            \\
        \hline
        \end{tabular}
        \label{table:main_results}
        % \begin{tablenotes}
        %     \item[1] CDF of the delay is bounded from below by $\alpha$-Pareto distribution CDF.
        %     \item[2] delay is $0$ with probability $p$, and $\infty$ otherwise. 
        %     \item[3] the delay distribution might depend on the reward.
        % \end{tablenotes}
    %\end{threeparttable}
\end{table}

\subsection{Related work}

To the best of our knowledge, \citet{dudik2011efficient} were the first to consider delays in stochastic MAB. They examine contextual bandit with fixed delay $d$, and obtain regret bound of $O(\sqrt{K\log (NT)}(d + \sqrt{T}) )$, where $N$ is number of possible policies. \citet{joulani2013online} use a reduction to non-delayed MAB. For their explicit bound they assume that expected value of the delay is bounded (see \Cref{table:main_results} for their \textit{implicit} bound). \citet{pike2018bandits} consider a more challenging setting in which the learner observe the \textit{sum} of rewards that arrive at the same round. They assume that the expected delay is known, and obtain similar bound as \citet{joulani2013online}.

\citet{vernade2017stochastic} study \textit{partially observed feedback} where the learner cannot distinguish between reward of $0$ and a feedback that have not returned yet, which is a special form of reward-dependent delay. However, they assume bounded expected delay and full knowledge on the delay distribution. 
\citet{manegueu2020stochastic} also consider partially observed feedback, and aim to relax the bounded expected delay assumption. They consider delay distributions that their CDF are bounded from below by the CDF of an $\alpha$-Pareto distribution, which might have infinite expected delay for $\alpha \leq 1$.
%\yishay{Check! I think it is $\alpha\leq 1$} 
However, this assumption still limits the distribution, e.g., the commonly examined fixed delay falls outside their setting. Moreover, they assume that the parameter $\alpha$ is known to the learner. Other extensions include Gaussian Process Bandit Optimization \cite{desautels2014parallelizing} and linear contextual bandits \cite{zhou2019learning}.
As opposed to most of these works, we place no assumptions on the delay distribution, and the learner has no prior knowledge on it.

Delays were also studied in the context of the non-stochastic MAB problem~\citep{auer2002nonstochastic}.
Generally, when reward are chosen in an adversarial fashion, the regret increases by a multiplicative factor of the delay.
Under full information, \citet{weinberger2002delayed} show regret bound of $O(\sqrt{dT})$, with fixed delay $d$. This was extended to bandit feedback by \cite{cesa2019delay}, with near-optimal regret bound of $O(\sqrt{T(K+d)})$. Several works have studied the effect of adversarial delays, in which the regret scales with $O(\sqrt{T}+\sqrt{D})$, where $D$ is the sum of delays \cite{thune2019nonstochastic,bistritz2019online,zimmert2020optimal,gyorgy2020adapting}. For last, \citet{cesa2018nonstochastic} consider a similar setting to \citet{pike2018bandits}, in which the learner observe only the sum of rewards. The increase in the regret is by a multiplicative factor of $\sqrt{d}$.

\section{Problem Setup and Background}
%\tal{Please review the whole section}
%\yishay{I think we assume that $r_t \in[0,1]$. we should add it}

We consider a variant of the classical stochastic Multi-armed Bandit (MAB) problem. 
In each round $t=1,2,\ldots,T$, an agent chooses an arm $a_{t}\in\left[K\right]$
and gets reward $r_{t}(a_t)$, where $r_t(\cdot) \in [0,1]^K$ is a random vector. Unlike the standard MAB setting, the agent does not immediately observe $r_{t}(a_t)$ at the end of round $t$; rather,
only after $d_{t}(a_t)$ rounds (namely, at the end of round $t+d_{t}(a_t)$) the tuple $(a_t, r_t(a_t))$ is received as feedback. 
We stress that neither the delay $d_t(a_t)$ nor the round number $t$ are observed as part of the feedback (so that the delay cannot be deduced directly from the feedback). The delay is supported in $\N \cup \{\infty\}$.
In particular, we allow $d_t(a_t)$ to be infinite, in which case the associated reward is never observed. 
The pairs of vectors $\{(r_t(\cdot), d_t(\cdot))\}_{t=1}^T$ are sampled i.i.d from a \textit{joint} distribution. Throughout the paper we sometimes abuse notation and denote $r_t(a_t)$ and $d_t(a_t)$ simply by $r_t$ and $d_t$, respectively.
This protocol is summarized in Protocol~\ref{protocol:stochastic_delays}.

We discuss two forms of stochastic delays: 
\emph{(i)} reward-independent delays, where the vectors $r_t(\cdot)$ and $d_t(\cdot)$ are \textit{independent} from each other,
and
\emph{(ii)} reward-dependent delays, where there is no restriction on the joint distribution.

The performance of the agent is measured as usual by the 
% As in most of the literature on stochastic MAB, we measure the performance of an algorithm by 
the difference between the algorithm's cumulative expected reward and the best possible total expected reward of any fixed arm. This is known as the \textit{expected pseudo regret}, formally defined by
% \begin{align*}
%     \mathcal{R}_{T} 
%     &=
%     % \textstyle\sum\limits
%     \max_{i}\mathbb{E}\left[\textstyle\sum\limits_{t=1}^{T} r_{t}(i)\right] - \mathbb{E}\left[\textstyle\sum\limits_{t=1}^{T}r_{t}(a_t)\right] 
%     \\
%     &= 
%     T\mu_{i^*} - \mathbb{E}\left[\textstyle\sum\limits_{t=1}^{T}\mu_{a_{t}}\right] = \mathbb{E}\left[\textstyle\sum\limits_{t=1}^{T}\Delta_{a_t}\right]
%     ,
% \end{align*}
\begin{align*}
    \mathcal{R}_{T} 
    &=
    % \textstyle\sum\limits
    \max_{i}\mathbb{E}\left[\sum_{t=1}^{T} r_{t}(i)\right] - \mathbb{E}\left[\sum_{t=1}^{T}r_{t}(a_t)\right] 
    = 
    T\mu_{i^*} - \mathbb{E}\left[\sum_{t=1}^{T}\mu_{a_{t}}\right] = \mathbb{E}\left[\sum_{t=1}^{T}\Delta_{a_t}\right]
    ,
\end{align*}
% \begin{align*}
%     \mathcal{R}_{T} 
%     &=
%     \max_{i}\mathbb{E}\left[\textstyle\sum\limits_{t=1}^{T} r_{t}(i)\right] - \mathbb{E}\left[\textstyle\sum\limits_{t=1}^{T}r_{t}(a_t)\right] 
%     \\
%     &= 
%     T\mu_{i^*} - \mathbb{E}\left[\textstyle\sum\limits_{t=1}^{T}\mu_{a_{t}}\right] = \mathbb{E}\left[\textstyle\sum\limits_{t=1}^{T}\Delta_{a_t}\right]
%     ,
% \end{align*}
 where $\mu_i$ is the mean reward of arm $i$, $i^*$ denotes the optimal arm and $\Delta_{i} = \mu_{i^*} - \mu_{i}$ for all $i \in [K]$. 
 %That is, $\mu_{i^*}=\max_{i}\mu_i$ and $\Delta_{i} = \mu_{i^*} - \mu_{i}$ for all $i \in [K]$.
%  \tomer{notations are not consistent: arms are $i$ or $a$?}.
%  \shahar{Notation is that actions taken are $a$ and arms are $i$. Maybe we should merge the two.}
 
%  Our regret bounds depends on the quantile function of the delay distribution. For arm $i$ we denote by $d_i(q)$ the $q$ quantile; formally, the delay of arm $i$ is a random variable $D_i$ for which
% % let $D_i$ be a sample for $i$'s delay distribution, 
% the quantile function is defined as %\shahar{We also use lambda for confidence bound}
% $$
%     d_i(q) = \min \big\{\gamma \in \N  \mid  \Pr[D_i \leq \gamma] \geq q \big\}.
% $$
  \setlength{\textfloatsep}{12pt}
 \begin{algorithm}[t]
    \floatname{algorithm}{Protocol}
    \caption{\label{protocol:stochastic_delays} 
             MAB with stochastic delays}
    \begin{algorithmic}
        \FOR{$t\in \left[T\right]$}
        
            \STATE Agent picks an action $a_t \in [K]$.
            \STATE Environment samples a pair, $(r_t(\cdot),d_t(\cdot))$, from a joint distribution.
            \STATE Agent get a reward $r_t(a_t)$ and observes feedback $\left\{(a_s,r_s(a_s)) : t = s + d_s(a_s) \right\}$.
        \ENDFOR
    \end{algorithmic}
    % \vskip -0.1in
\end{algorithm}

For a fixed algorithm for the agent (the relevant algorithm will always be clear from the context), we denote by $m_t(i)$ the number of times it choose arm $i$ by the end of round $t-1$. 
Similarly $n_{t}(i)$ denotes the number of \textit{observed} feedback from arm $i$, by the end of round $t-1$. The two might differ as some of the feedback is delayed. 
Let~$\hat{\mu}_t(i)$ be the \textit{observed} empirical average of arm $i$, defined as:
$$
    \hat{\mu}_t(i) = \frac{1}{n_t(i) \vee 1} \sum_{s :s + d_s < t} \indicator\{a_s = i\} r_s,
$$
where $a \vee b = \max\{a,b\}$ and $\indicator\{\pi\}$ is the indicator function of predicate $\pi$.
% \tomer{this is weird - all these definitions depend on the agent's algorithm...}

We denote $d_i(q)$ to be the quantile function for arm $i$'s delay distribution; formally, if $D_i$ is the delay of arm $i$ then the quantile function is defined as %\shahar{We also use lambda for confidence bound}
\begin{align*}
    d_i(q) = \min \big\{\gamma \in \N  \mid  \Pr[D_i \leq \gamma] \geq q \big\}.
\end{align*}
% Finally, we state a simple concentration bound for the estimation of the expected rewards, which follows immediately from Hoeffding's inequality and a union bound. 

% \begin{lemma}
%     \label{lemma:estimator_bound}
%     Let $\hat{\mu}_t(i)$ be the observed empirical average of the expected reward $\mu_i$ at the beginning of step $t$. Then
%     $$
%         \Pr \Big[ \exists ~ i, t \;:\; \abs{\hat{\mu}_t(i) - \mu_i} > \sqrt{\frac{2\log{T}}{n_t(i)}}\Big] \leq \frac{2}{T^2}.
%     $$
% \end{lemma}
% \tal{we don't ref to this lemma. Perhaps we should move it to the appendix?}

\section{Reward-independent Delays}
\label{sec:rew-ind}
% In this setting we let the delays to be stochastic, but independent of the rewards. Each arm $i \in \left[K\right]$, has its own delay distribution. At time $t$, after the agent chooses arm $a_t$, the delay $d_t$ is sampled from the respected arm distribution. The agent will observe that reward at the end of round $t + d_t$. If $t + d_t > T$, the the agent will never observe the reward. This is very likely since the delay might be infinite. We note that when the agent observes a reward, she does not know the time she received it, only the action she played at that time. A formal description appears in Protocol \ref{protocol:stochastic_delays}. 

% *** TAL: changed the above as we repeated ourselves ***
We first consider the case where delays are independent of the realized stochastic rewards.
%first setting, in which delays are independent of the reward.
We begin with an analysis of two classic algorithms: UCB~\citep{auer2002finite} and Successive Elimination (SE)~\citep{even2006action}, adjusted to handle delayed feedback in a straightforward naive way (see Procedure~\ref{alg:updates}). 
% In Procedure~\ref{alg:updates}, we let $n_{t}(i)$ be the number of \textit{observed} feedback from arm $i$, by the end of round $t-1$ and $\hat{\mu}_t(i)$ is defined as the \textit{observed} empirical average of arm $i$, formally:
% $$
%     \hat{\mu}_t(i) = \frac{1}{n_t(i) \vee 1} \sum_{s :s + d_s < t} \indicator\{a_s = i\} r_s,
% $$
% where $a \vee b = \max\{a,b\}$. 

% In \Cref{sec:UCB}, we show that even under a 
% %constant \tomer{what does `constant' mean in this context?} 
% fixed delay UCB performs surprisingly poorly. In contrast, in \cref{sec:SE} we show that SE performs significantly better under general delay distributions. 
% Later, in \cref{sec:phased-SE}, we present a new algorithm, Phased Successive Elimination (PSE), which is based on SE. PSE enjoys a similar regret bound, but unlike SE, it is not affected by the delay distribution of the optimal arm. Lastly, in \cref{sec:lower_bound_reward_indep} we show a lower bound for any algorithm in this setting, under standard instance-dependent assumptions on the regret of that algorithm. 
% % \tal{I don't recall what we've settled. Do you think we should change this paragraph?}
% \tomer{if we want to save space - this paragraph is not crucial} \shahar{agree}

\subsection{Suboptimality of UCB with delays}  \label{sec:UCB}

UCB \cite{auer2002finite} is based on the basic principle of ``optimism under uncertainty.'' It maintains for each arm an upper confidence bound (UCB): a value that upper bounds the true mean with high probability. In each round it simply pulls the arm with the highest UCB. The exact description appears in~\cref{alg:ucb}.

\begin{algorithm}[ht]
    \floatname{algorithm}{Procedure}
    \caption{Update-Parameters}   \label{alg:updates}
    \begin{algorithmic}
        % \STATE{\color{gray} \# begin with sampling each arm once}
        \FOR{$i \in [K]$}
            %\STATE {\color{gray} \# number of observed feedback}
            \STATE $n_{t}(i) \gets \sum_{s:s+d_{s} < t}\indicator\{a_{s}=i\} $ 
            \hfill {\color{gray} \# number of observed feedback}
            %\STATE {\color{gray} \# observed empirical mean}
            \STATE $\hat{\mu}_{t}(i) \gets \frac{1}{n_{t}(i) \vee 1}\sum_{s:s+d_{s}< t}\indicator\{a_{s}=i\}r_{s}$ 
            \hfill {\color{gray} \# observed empirical mean}
            \STATE  $LCB_{t}(i) \gets \hat{\mu}_{t}(i)-\sqrt{\frac{2\log T}{n_{t}(i) \vee 1}}$
            \STATE  $UCB_{t}(i) \gets \hat{\mu}_{t}(i)+\sqrt{\frac{2\log T}{n_{t}(i) \vee 1}}$
            % \tomer{should be UCB and LCB no? I would just find a nice mathematical notation for these...}
            % \shahar{Fixed typos. The convention is UCB/LCB right? Is there a shorter acceptable notation?}
        \ENDFOR
    \end{algorithmic}
\end{algorithm}

\begin{algorithm}[ht]
    \caption{UCB with Delays}   \label{alg:ucb}
    \begin{algorithmic}
        \STATE \textbf{Input:} number of rounds $T$, number of arms $K$
        \STATE \textbf{Initialization:} $t \leftarrow 1$
        \STATE{\color{gray} \# begin with sampling each arm once}
        \STATE Pull each arm $i \in [K]$ once
        \STATE Observe any incoming feedback %\tomer{what about delays??}
        \STATE Set $t \leftarrow t + K$
        
        \WHILE{$t < T$}
            \STATE Call \textit{Update-Parameters} (Procedure \ref{alg:updates})
            \STATE Pull arm $a_t \in \argmax_i UCB_t(i)$ 
            \hfill {\color{gray} \# With deterministic tie breaking rule i.e. by index}
            % \STATE {\color{gray} \# ties are broken by deterministic tie breaking rule i.e. by index}
            %\STATE {\color{gray} \# With deterministic tie breaking rule i.e. by index}
            % \tomer{then you have to say what that is...}
            \STATE Observe feedback $\left\{(a_s, r_s) : s + d_s = t\right\}$
            \STATE Set $t \leftarrow t + 1$
            % \STATE Call \textit{Update-Parameters} (Procedure \ref{alg:updates}) %\tomer{with $K$?}
            % \STATE \yishay{computing $m_t$ and $n_t$ is highly different and very important to have in the code! Why not simply have initial UCB to be $2\log T$?} \shahar{It won't be truthful to UCB in case the feedback for some arms do arrive until the end of the round rob. I've added a generic procedure instead like you adviced below.}
        \ENDWHILE
    \end{algorithmic}
\end{algorithm}
% \begin{algorithm}[t]
%     \caption{UCB with Delays}   \label{alg:ucb}
%     \begin{algorithmic}
%         \STATE \textbf{Input:} number of rounds $T$, number of arms $K$
%         \STATE \textbf{Initialization:} $t \leftarrow 1$
%         \STATE{\color{gray} \# begin with sampling each arm once}
%         \FOR{$i \in [K]$}
%             \STATE Pull arm $i$
%             \STATE Observe feedback $\left\{(a_s, r_s) : s + d_s = t\right\}$
%             \STATE Set $t \leftarrow t + 1$
%         \ENDFOR
%         \STATE Call \textit{Update-Parameters} with $[K]$ (Procedure \ref{alg:updates})
%         \WHILE{$t < T$}
%             \STATE Pull arm $a_t \in \argmax_i UCB_t(i)$ 
%             \STATE {\color{gray} \# ties are broken by deterministic tie breaking rule}
%             \STATE Observe feedback $\left\{(a_s, r_s) : s + d_s = t\right\}$
%             \STATE Set $t \leftarrow t + 1$
%             \STATE Call \textit{Update-Parameters} with $[K]$ (Procedure \ref{alg:updates})
%             % \STATE \yishay{computing $m_t$ and $n_t$ is highly different and very important to have in the code! Why not simply have initial UCB to be $2\log T$?} \shahar{It won't be truthful to UCB in case the feedback for some arms do arrive until the end of the round rob. I've added a generic procedure instead like you adviced below.}
%         \ENDWHILE
%     \end{algorithmic}
% \end{algorithm}

In the standard non-delayed setting, UCB is known to be optimal. However, with delays this is no longer the case. Consider the simpler case where all arms suffers from a constant fixed delay $d$. \citet{joulani2013online} show that the regret of UCB with delay is bounded by $O(\mabregret + Kd)$. We show that the increase in the regret is necessary for UCB, and the additional regret due to the delay can in general scale as $\Omega(Kd)$. 
The reason is due to the nature of UCB: it always samples the currently most promising arm,
and it might take as much as $d$ rounds to update the latter.
% An update of which arm is most promising might take as much as $d$ rounds.
%The reason for the difference in regret between the two algorithms is as follows: while waiting for a feedback from an arm, SE keeps sampling the rest of the active arms. Whereas, UCB might keep sampling the same arm for many contiguous rounds. That way, while the number of rounds that SE runs before it observes $m$ samples for $K$ arms is approximately $Km+d$, UCB might require in certain cases $K(m+d)$ rounds.
%
This is formalized in the following theorem (proof is deferred to~\cref{appendix:ucb-lower}.)

\begin{theorem}     \label{thm:ucb-lower}
    Under fixed delay $d \geq K$, there exist a problem instance such that $UCB$ suffers regret of $\Omega(Kd)$.
    % \[
    %     \Rcal_T \geq \Omega(Kd)
    % \]
\end{theorem}

% In the next section we present a variant of the SE algorithm in which the increase in regret due to the delays does not have the unfavorable dependence on $K$ of UCB.

% % *********** SKETCH ***************
% \begin{proof}[Proof of \Cref{thm:ucb-lower} (sketch)]
%     Consider an instance in which all rewards are sampled from Bernoulli distributions, where the optimal arm has $\mu_{i^{*}} = 1$ and the rest of the arms are with mean $\mu_i=\frac{1}{2}$. Without loss of generality we assume that the tie breaking rule is by index and that $i^* \geq K/2$.\footnote{We can always randomize the index of the optimal arm. In this case, the assumption holds with probability of at least $1/2$, which affects the regret only by a constant.} 
%     With high probability, a constant fraction of arms $i < i^*$ has higher UCB then the optimal arm when it is calculated with respect to at most one sample. After sampling each arm once, the algorithm will sample each of those arms for at least $d$ consecutive rounds, as the UCB does not change until the second feedback is observed. Hence UCB pulls for at least $\Omega(Kd)$ rounds a sub-optimal arm.
% \end{proof}

% Combining the above with the known lower bound for UCB in the non-delayed setting gives us the following corollary:
% \begin{corollary}
%     Consider 
% \end{corollary}
% \todo{modify the above for any $\Delta_i$ and add a corollary that combines the standard lower bound for UCB}

\subsection{Successive Elimination with delays} \label{sec:SE}

Successive Elimination (SE) maintains a set of active arms, where initially all arms are active. It pulls all arms equally and whenever there is a high-confidence that an arm is sub-optimal, it eliminates it from the set of active arms. The exact description appears in \cref{alg:non-phased-SE}. 

\begin{algorithm}
    \caption{Successive Elimination with Delays} \label{alg:non-phased-SE}
    \begin{algorithmic}
    \STATE \textbf{Input:} number of rounds $T$, number of arms $K$
    \STATE \textbf{Initialization:} $S \gets [K]$, $t \gets 1$
    \WHILE{$t < T$}
        \STATE Pull each arm $i \in S$
        \STATE Observe any incoming feedback
        \STATE Set $t \leftarrow t + |S|$
        \STATE Call \textit{Update-Parameters} (Procedure \ref{alg:updates})
        \STATE{\color{gray} \# Elimination Step}
        \STATE Remove from $S$ all arms $i$ such that exists $j$ with $UCB_t({i})<LCB_t({j})$
    \ENDWHILE
    \end{algorithmic}
\end{algorithm}
% \begin{algorithm}  
%     \caption{Successive Elimination with Delays} \label{alg:non-phased-SE}
%     \begin{algorithmic}
%     \STATE \textbf{Input:} number of rounds $T$, number of arms $K$
%     \STATE \textbf{Initialization:} $S=\left\{ 1,...,K\right\} $, $t \leftarrow 1$
%     \WHILE{$t < T$}
%         \FOR{$i \in S$}
%             \STATE Pull arm $a_t = i$
%             \STATE Observe feedback $\{(a_s,r_s): s+d_s = t\}$
%             \STATE Set $t \leftarrow t + 1$
%         \ENDFOR
%         \STATE Call \textit{Update-Parameters} with $S$ (Procedure \ref{alg:updates})
%             % \STATE\yishay{Maybe have the above four lines as a procedure UPDATE} \shahar{great idea!!}
%         \STATE{\color{gray} \# Elimination Step}
%         \FOR{any $i, j\in S$ such that $UCB_t({i})<LCB_t({j})$}
%             \STATE $S=S\backslash\left\{ i\right\} $
%         \ENDFOR
%     \ENDWHILE
%     \end{algorithmic}
% \end{algorithm}

% Here we show that SE only suffers an additive $O(d)$ for fixed delays. The difference from 
Unlike UCB, SE continues to sample all arms equally, and not just the most promising arm. In fact, the number of rounds that SE runs before it observes $m$ samples for $K$ arms is approximately $Km+d$, whereas UCB might require $K(m+d)$ rounds in certain cases. 
More generally, we prove:
% \shahar{We need to better introduce the whole qunatile approach - either here or in the intro. because it is not trivial to think in quantile terms about this question.} \tal{was added in the intro}
\begin{theorem}
    \label{thm:SE_arm_quantiles}
    For reward-independent delay distributions,
    % $\{\mathcal{D}_i\}_{i=1}^K$, 
    the expected pseudo-regret of \cref{alg:non-phased-SE} is bounded by
    % \begin{equation}
    %     \label{eq:regret_SE_single_quantile}
    %     \begin{aligned}   
    %     \mathcal{R}_{T} \leq   \min_{q \in (0,1]} & \sum_{i\ne i^{*}}\frac{544\log(T)}{q\Delta_i} + 2\max_{i}d_i(q).
    %     \end{aligned}
    % \end{equation}
    % Additionally, minimization of \cref{eq:regret_SE_single_quantile} over a vector $\vec{q} \in (0,1]^K$ yields, 
    %  \begin{equation}
    %     \label{eq:SE_arms_quantiles}
    %     \begin{aligned}   
    %     \mathcal{R}_{T} \leq   &\min_{\vec{q} \in (0,1]^K}  \sum_{i\ne i^{*}}\frac{40\log{T}}{\Delta_i}\left(\frac{1}{q_i}+\frac{1}{q_{i^{*}}}\right)
    %     \\
    %     & \quad + \log{(K)} \max_{i \neq i^*}\big\{ (d_i(q_i)+d_{i^*}(q_{i^*}))\Delta_i \big\} 
    %     .
    %     \end{aligned}
    % \end{equation}
     \begin{equation}
        \label{eq:SE_arms_quantiles}
        \begin{aligned}   
        \mathcal{R}_{T} \leq   &\min_{\vec{q} \in (0,1]^K}  \sum_{i\ne i^{*}}\frac{40\log{T}}{\Delta_i}\left(\frac{1}{q_i}+\frac{1}{q_{i^{*}}}\right)
         + \log{(K)} \max_{i \neq i^*}\big\{ (d_i(q_i)+d_{i^*}(q_{i^*}))\Delta_i \big\} 
        .
        \end{aligned}
    \end{equation}
    Additionally, if instead we minimize over a single quantile $q \in (0,1]$, the expected pseudo-regret becomes
    \begin{equation}
         \label{eq:SE_single_quantile}
         \mathcal{R}_{T} \leq \min_{q \in (0,1]} \sum_{i \neq i^*} \frac{325\log\left(T\right)}{q \Delta_i} + 4\max_{i \in [K]}d_i(q).
    \end{equation}
    % \tomer{the final $\max$ is over all $i$ or $i \neq i^*$? be more precise} \shahar{all $i$}
    % \begin{equation}
    %      \label{eq:SE_fixed_delay}
    %     \mathcal{R}_{T} \leq \sum_{i \neq i^*}\frac{325\log\left(T\right)}{ \Delta_i} + 4d.
    % \end{equation}
\end{theorem}

% \begin{remark}
%     \label{remark:SE_single_quantile}
%     In case that the minimization in \cref{eq:regret-non-phased-SE} is taken over a scalar $q \in (0,1]$ across all arms, we are able to eliminate the $\log(K)$ factor on the second term. We provide the formal statement and its proof in \Cref{appendix:reward-independent-single-q} 
% \end{remark}

Particularly, \cref{thm:SE_arm_quantiles} implies that for fixed delay $d$, we have $\mathcal{R}_{T} = O(\mabregret + d)$. 
Note that the bounds in \cref{eq:SE_arms_quantiles,eq:SE_single_quantile} are incomparable:  \cref{eq:SE_arms_quantiles} allows choosing a different quantile for each arm, while \cref{eq:SE_single_quantile} gives a slightly better dependence on $K$.

We now turn to show the main ideas of the proof of \cref{thm:SE_arm_quantiles}, deferring the full proof to~\cref{appendix:proof_SE_arm_quantiles}.
% The full proof of \Cref{thm:SE_arm_quantiles} appears in \Cref{appendix:proof_SE_arm_quantiles}. 

\begin{proof}[Proof of \Cref{thm:SE_arm_quantiles} (sketch)]
    Here we sketch the proof of~\cref{eq:SE_arms_quantiles}; proving \cref{eq:SE_single_quantile} is similar, but requires a more delicate argument in order to eliminate the $K$ dependency in the second term. 

    Fix some vector $\vec{q}\in(0,1]^K$ and let $d_{max} = \max_{i \ne i^*}d_i(q_i)$. First, with high probability all the true means of the reward remain within the confidence interval (i.e., $\forall t,i:\mu_{i}\in[LCB_{t}(i),UCB_{t}(i)]$). Under this condition, the optimal arm is never eliminated. If a sub-optimal arm $i$ was not eliminated by time $t$ then, $LCB_{t}({i^{*}}) \leq UCB_{t}(i)$.
    % \[
    %      LCB_{t}({i^{*}}) \leq UCB_{t}(i),
    % \]
    which implies with high probability,
    \begin{align*}
        \Delta_i = \mu_{i^{*}} - \mu_i 
        \leq 
        2\sqrt{\frac{2\log(T)}{n_{t}(i)}} + 2\sqrt{\frac{2\log(T)}{n_{t}(i^*)}}.
    \end{align*}  
    % Now, using \Cref{lemma:quantile_bound_chernoff}, 
    % we bound from below $n_t(i)$ and $n_t(i^*)$ to obtain,
    Now, using a concentration bound, we show that the amount of \emph{observed} feedback from arm $j$ at time $t$, is approximately a fraction $q_j$ of the number of pulls at time $t-d_j(q_j)$. We use that to bound $n_t(i)$ and $n_t(i^*)$ from below and obtain,
    \[
        m_{t-d_{max}}(i)
        = 
        O\left(
            \frac{\log T}{\Delta_{i}^{2}}\left(\frac{1}{q_{i}}+\frac{1}{q_{i^{*}}}\right)
        \right)
        .
    \]
    Now, if $t$ is the last time we pulled arm $i$, then we can write the total regret from arm $i$ as,
    \begin{align*}
        m_{t}(i)\Delta_{i} 
        &=
        m_{t-d_{max}}(i)\Delta_{i}+(m_{t}(i)-m_{t-d_{max}}(i))\Delta_{i}\\
        &\leq
        O \left(\frac{\log T}{\Delta_{i}}\left(\frac{1}{q_{i}}+\frac{1}{q_{i^{*}}}\right)\right) + m_{t}(i)-m_{t-d_{max}}(i).
    \end{align*}
    % \shahar{We lost the $\Delta_i$ factor, is that for simplification?}
    The difference $m_t(i) - m_{t-d_{max}}(i)$ is number of times we pull $i$ between time $t-d_{max}$ and $t$. This is trivially bounded by $d_{max}$, but since we round-robin over active arms, we can divide it by the number of active arms. At the first elimination there are $K$ active arms, in second there $K-1$ active arms, and so on. When summing the regret of all arms we get,
    \begin{align*}
        \mathcal{R}_{T}
        &= 
        O\left(\sum_{i\ne i^*}
            \frac{\log T}{\Delta_{i}}\left(\frac{1}{q_{i}}+\frac{1}{q_{i^{*}}}\right)
        \right)
        + \log(K)d_{max}\\
        &=
        O\left(\sum_{i\ne i^*}
            \frac{\log T}{\Delta_{i}} \left(\frac{1}{q_{i}} + \frac{1}{q_{i^{*}}}\right)
        \right) + \log(K)\max_{i}d_{i}(q_{i}),
    \end{align*}
    where we have used the fact that $1/K+1/(K-1)+...+1/2\leq\log(K)$. This proves the bound in \cref{eq:SE_arms_quantiles}.
\end{proof}

\subsection{Phased Successive Elimination}
\label{sec:phased-SE}

Next, we introduce a phased version of successive elimination, we call Phased Successive Elimination (PSE). Inspired by phased versions of the commonly used algorithms \cite{auer2010ucb}, the algorithm works in phases. Unlike SE, it does not round-robin naively, instead it attempts to maintains a balanced number of \textit{observed} feedback at the end of each phase. As a result, PSE does not depend on the delay of the optimal arm. Surprisingly, the dependence on the delay of the sub-optimal arms remain similar, up to log-factors. %\tomer{why? give a short sentence to add context} 

On each phase $\ell$ of PSE, we sample arms that were not eliminated in previous phase in a round-robin fashion. When we observe at least $16\log(T)/2^{-2\ell}$ samples for an active arm, we stop sampling it, but keep sampling the rest of active arms. Once we reach enough samples from all active arms, we perform elimination the same way we do on SE, and advance to the next phase $\ell + 1$. The full description of the algorithm is found in \cref{alg:phased-SE}.

\begin{algorithm}[h]
    \caption{Phased Successive Elimination (PSE)} \label{alg:phased-SE}
    \begin{algorithmic}
    \STATE \textbf{Input:} number of rounds $T$, number of arms $K$
    \STATE \textbf{Initialization:} $S \gets [K],~ \ell \leftarrow 0,~ t \leftarrow 1 $
    \WHILE{$t < T$}
        \STATE Set $\ell \leftarrow \ell + 1$ (phase counter)
        \STATE Set $S_{\ell} \gets S$
        \WHILE{ $S_{\ell} \neq \emptyset$}
            \STATE Pull each arm $i \in S_{\ell}$, observe incoming feedback
            \STATE Set $t \leftarrow t + |S_{\ell}|$
            % \FOR{$i \in S_{\ell}$}
            %     \STATE Pull arm $a_t = i$
            %     \STATE Observe feedback $\{(a_s,r_s): s+d_s = t\}$
            %     \STATE Set $t \leftarrow t + 1$
            % \ENDFOR
            \STATE Call \textit{Update-Parameters} (Procedure \ref{alg:updates})
            % \FOR{any $i\in S_{\ell}~~$  if  $~~n_{t}(i)\geq 16\log(T)/2^{-2\ell}~~$}
            %     \STATE  $S_{\ell}=S_{\ell}\backslash\left\{ i\right\} $
            % \ENDFOR
            \STATE Remove from $S_\ell$ all arms that where observed at least $16\log(T)/2^{-2\ell}$ times.
            % \FOR{$i \in S_{\ell}$}
            %     \STATE Compute $n_{t}(i) =\sum_{s:s+d_{s}\leq t}\indicator\{a_{s}=i\} $
            %     \IF{$n_{t}(i)\geq 16\log(T)/2^{-2\ell}$}
            %         \STATE $S_{\ell}=S_{\ell}\backslash\left\{ i\right\} $
            %     \ENDIF
            % \ENDFOR
        \ENDWHILE
        % \STATE Call \textit{Update-Parameters} with $S$ (Procedure \ref{alg:updates})
        % \STATE{\color{gray} \# Elimination Step}
        % \FOR{any $i, j\in S$ such that $UCB_t({i})<LCB_t({j})$}
        %     \STATE $S=S\backslash\left\{ i\right\} $
        % \ENDFOR
        \STATE Remove from $S$ all arms $i$ such that exists $j$ with $UCB_t({i})<LCB_t({j})$
    \ENDWHILE
    \end{algorithmic}
\end{algorithm}

\begin{theorem}
    \label{thm:PSE_reward_indep}
    For reward-independent delay distributions,
    %$\{\mathcal{D}_i\}_{i=1}^K$, 
    the expected pseudo-regret of \cref{alg:phased-SE} (PSE) satisfies
    % \begin{equation}
    %     \label{eq:regret_PSE_single_quantile}
    %     \begin{aligned}
    %     \mathcal{R}_{T}  \leq  \min_{q \in (0,1]} & \sum_{i \neq i^*}\Big(\frac{514\log\left(T\right)}{q\Delta_i}\Big) +2\max_{i \neq i^*}d_i(q).
    %     \end{aligned}
    % \end{equation}
    % Additionally, minimization of \cref{eq:regret_PSE_single_quantile} over a vector $\vec{q} \in (0,1]^K$ yields, 
    % for any $\vec q \in (0,1]^K$, as \tomer{how is this? shorter and looks better imo (original version in comments, feel free to revert}
    % \begin{equation} \label{eq:regret_PSE}
    %     % \begin{aligned}
    %         \mathcal{R}_{T}  
    %         \leq  
    %         \sum_{i \neq i^*} \frac{290 \log{T}}{q_i\Delta_i}
    %         + \log(T)\log(K)\max_{i \neq i^*} d_i(q_i)\Delta_i
    %         .\;
    %     % \end{aligned}
    % \end{equation}
    \begin{equation}
        \label{eq:regret_PSE}
        \begin{aligned}
            \mathcal{R}_{T}  
            \leq  
            &\min_{\vec{q} \in (0,1]^K}  \sum_{i \neq i^*} \frac{290 \log{T}}{q_i\Delta_i}
            + \log(T)\log(K)\max_{i \neq i^*} d_i(q_i)\Delta_i.
        \end{aligned}
    \end{equation}
\end{theorem}
The proof of~\cref{thm:PSE_reward_indep} appears in~\cref{appendix:PSE-proof}.
Similarly to the proof \Cref{thm:SE_arm_quantiles}, both SE and PSE eliminate arm $i$ approximately whenever 
\[
    \sqrt{\frac{\log T}{n_{t}(i)}}+\sqrt{\frac{\log T}{n_{t}(i^{*})}} \approx \Delta_{i}.
\]%Keep as equation if possible
In a sense, PSE aims to shrink both terms in the left-hand side at a similar rate, which avoid the dependence on $1/q_{i^*}$ in the first term of \cref{eq:regret_PSE}. The down side is in the second term: SE keeps sampling all active arms at the same rate, which gives rise to the $\log(K)$ dependence in the second term. Under PSE this is no longer the case: naively, one could show a linear dependence on $K$, but a more careful analysis that uses round-robin sampling within phases gives a $\log(T)\log(K)$ dependence in the second term of \cref{eq:regret_PSE}.

One important example in which PSE dominates SE is the arm-dependent packet loss setting,  
where we get the feedback of arm $i$ immediately (i.e., zero delay) with probability $p_i$, 
% \tal{I it is better to change those to $p_i$, to deffer from the quantiles.}\shahar{Sure. Let's try both to see which one fits better.}
and infinite delay otherwise. The regret of SE in this setting is $O(\sum_{i\ne i^{*}}\log(T)/\Delta_{i}\cdot(\ifrac{1}{p_i}+\ifrac{1}{p_{i^{*}}}))$. 
On the other hand, PSE's regret is bounded by $O(\sum_{i\ne i^{*}}\log T/(\Delta_{i}p_{i}))$. 
The difference in the regret is substantial when $p_{i^*}$ is very small. In fact, small amount of feedbacks from the optimal arm only benefits PSE, as it would keep sampling it until it gets enough feedbacks.

\subsection{Lower Bound}
\label{sec:lower_bound_reward_indep}
% \tal{in all lower bounds, $\Delta$ should be  set and not a vector}
% \shahar{I believe this is now fixed}

We conclude this section with showing an instance-dependent lower bound (an instance is defined by the set of sub-optimality gaps $\Delta_i$). 
% We construct a delay distribution which is used across all arms. Specifically, let $q \in (0,1],~\Tilde{d} \leq T$. At time $t$, the delay is $\Tilde{d}$ with probability $q$ and $\infty$ otherwise. \tomer{something is missing here..}
% \yishay{Maybe give such distribution a name, since we use them a few times. Not sure.}
% \shahar{We use it here and in \cref{cor:0-infty-delay}, its not that much in my opinion. We can add it in the preliminaries perhaps}

% \begin{align*}
%     \Pr(d_{t} = d)
%     =
%     \begin{cases}
%         q   & d = \Tilde{d},\\
%         1-q & d = \infty.
%     \end{cases}
% \end{align*}

% \yishay{Need to compare to the upper bound}
% \shahar{We compare after we show the theorem, should we move it up?}
\begin{theorem} \label{thm:lower-bound-ind}
    Let $ALG^{delay}$ be an algorithm that guarantees a regret bound of $T^\alpha$ over any instance. For any sub-optimality gaps set $S_\Delta = \{\Delta_i : \Delta_i \in [0,\tfrac14]\}$ of cardinality $K$, a quantile $q\in(0,1]$, and $\Tilde{d} \leq T$, there exists an instance with an order on $S_\Delta$,
    %\tomer{it's very confusing to have $\Delta$ denote a set - usually it refers to the minimal gap. can't we just refer as $\Delta_1,\ldots,\Delta_K$?}\tal{If we fix the deltas in such way, then the learner would know which arm is the optimal one. It doesn't matter for the first term, but for the second we take a random order on the arms} \shahar{Maybe we'll just change the symbol?} 
    and delay distributions with $d_i(q) = \tilde{d}$ for any $i$, such that $ALG^{delay}$'s regret on that instance is
    \begin{align} \label{eq:lower-bound-ind}
        \mathcal{R}_T
        \geq
        \frac{1}{128}\sum_{i:S_\Delta \ni \Delta_i > 0}\frac{(1 - \alpha)\log T}{q\Delta_{i}}
        + \frac{1}{2}\Deltabar\max_{i \in [K]} d_i(q)
    \end{align}
    for sufficiently large $T$, % (that depends on $q, \Delta$ and $\alpha$)
    where $\Deltabar = \frac{1}{K} \sum_{i \in [K]} \Delta_i$. 
\end{theorem}  
The lower bound is proved using delay distribution which is homogeneous across all arms: at time $t$, the delay is $\Tilde{d}$ with probability $q$ and $\infty$ otherwise. The upper bound of SE and PSE involves a minimization over $q_i$. In this case, it is solved by $q_i = q$ for all $i$. Therefore, the best comparison is to \cref{eq:SE_single_quantile} in \cref{thm:SE_arm_quantiles}, where a single quantile is chosen.
\Cref{thm:lower-bound-ind} shows that SE is near optimal in this case. The first term in \cref{eq:SE_single_quantile} is aligned with \cref{eq:lower-bound-ind}, up to constant factors. The difference between the two is on the second term, where there is a $\Deltabar$ factor in the lower bound.

The second term in \cref{eq:lower-bound-ind} is due to the fact that the algorithm does not get any feedback for the first $\Tilde{d} = d_i(q)$ rounds. Thus, any order on $\Delta$ is statistically indistinguishable from the others for the first rounds. Therefore, the learner suffers $\Deltabar$ regret on average, over the first $\tilde{d}$ rounds, under at least one of the instances.
The first term is achieved using a reduction from instance-depended lower bound for MAB without delays~(\citealp{kleinberg2010regret}; see also \citealp{lattimore2020bandit}). The regret is bounded from below by this term, even if the instance $I$ is known to the learner (the regret guarantee over the other instances ensures that the algorithm does not specialized particularly for that instance). A more detailed lower bound and its full proof is provided in~\cref{appendix:lower-bound-ind}.

\section{Reward-dependent Delays}

We next consider the more challenging case where we let the reward and the delays to be probabilistically dependent. 
Namely, there is no restriction on the reward-delay joint distribution.

%*************** OLD *************
% We next consider the more challenging case where we let the reward and the delays to be dependent. At time $t$, after the agent chooses arm $a_t$ and received reward $r_t$, the delay $d_t$ is sampled from the respected \tomer{?} delay distribution, based on $a_t$ and $r_t$. \tomer{this needs to be more formal - here and in the problem setup}
% %The formal description appears in Protocol \ref{protocol:stochastic_delays}.
%*************** OLD *************

The main challenge in this setting is that the \textit{observed} empirical mean is no
longer an unbiased estimator of the expected reward; e.g., if the delay given a reward of $0$ is shorter than the delay given that the reward is $1$, then the observed empirical mean would be biased towards $0$. Therefore, the analysis from the previous section
does not hold anymore. 
To tackle the problem, we present a new variant of successive elimination, Optimistic-Pessimistic Successive Elimination (OPSE), described in \cref{alg:SE-reward-dependent}. When calculating UCB the agent is optimistic regarding the unobserved samples, by assuming all missing samples have the maximal reward (one). 
% We denote the \textit{biased} mean estimator of UCB by $\hat{\mu}_{t}^{+}(i)$. 
When calculating LCB the agent assumes all missing samples have the minimal reward (zero). % $\hat{\mu}_{t}^{-}(i)$ denotes the biased mean estimator of LCB. 
%The formal description appears in \cref{alg:SE-reward-dependent}. 
We emphasize that unlike the previous section, here the estimators take into account all samples, including the unobserved ones. The above implies that the confidence interval computed by OPSE contains the confidence interval computed by non-delayed SE.
% (see formal definition in~\cref{alg:SE-reward-dependent}). 
% Formally, we define two (biased) estimators: 
% \begin{align*}
%     \hat{\mu}_{t}^{-}(i)&=\frac{1}{m_{t}(i)}\sum_{s : s + d_s < t}\indicator\{a_s=i\}r_{s},
%     \\
%     \hat{\mu}_{t}^{+}(i)&= \frac{m_{t}(i)-n_{t}(i)}{m_{t}(i)}+\hat{\mu}_{t}^{-}(i)
%     \\
%     &= \frac{m_{t}(i)-n_{t}(i)}{m_{t}(i)}+\frac{1}{m_{t}(i)}\sum_{s : s + d_s < t}\indicator\{a_s=i\}r_{s}.
% \end{align*}
% Now, the upper and lower confidence bound are defined by,
% \begin{align*}
%     &UCB_{t}(i) = \hat{\mu}_{t}^{+}(i) + \sqrt{\frac{2\log T}{m_{t}(i)}},
%     \\
%     &LCB_{t}(i) = \hat{\mu}_{t}^{-}(i) - \sqrt{\frac{2\log T}{m_{t}(i)}}.
% \end{align*}

% \yishay{We probably should say why are we using SE and not PSE}
% \shahar{Added after introducing the theorem.}

\begin{algorithm}[h]
    \caption{Optimistic-Pessimistic Successive Elimination}
    \begin{algorithmic}     \label{alg:SE-reward-dependent}
        \STATE \textbf{Input:} number of rounds $T$, number of arms $K$
        \STATE \textbf{Initialization:} $S \gets [K] $, ~$t\leftarrow1$
        \WHILE{$t < T$}
            \STATE Pull each arm $i \in S$
            \STATE Observe any incoming feedback
            \STATE Set $t \leftarrow t + |S|$
            \FOR{$i \in S$}
                % \STATE {\color{gray} \# the number of pulls and observations}
                \STATE $m_{t}(i) \gets \sum_{s < t}\indicator\{a_{s}=i\} $ 
                \hfill {\color{gray} \# number of pulls}
                % {\color{gray} \# the number of pull}
                % (the number of pull)
                % \STATE {\color{gray} \# the number of observation}
                \STATE $n_{t}(i) \gets \sum_{s:s+d_{s}< t}\indicator\{a_{s}=i\} $ 
                \hfill {\color{gray} \# number of observations}
                % {\color{gray} \# the number of observation}
                %(the number of observation)
                % \STATE {\color{gray} \# pessimistic and optimistic estimators for $\mu_i$}
                \STATE $\hat{\mu}_{t}^{-}(i) \gets \frac{1}{m_{t}(i)}\sum_{s:s+d_{s}< t}\indicator\{a_{s}=i\}r_{s}$
                \hfill {\color{gray} \# pessimistic estimator for $\mu_i$}
                % {\color{gray} \# pessimistic estimator for $\mu_i$}
                % \STATE {\color{gray} optimistic estimator for $\mu_i$}
                \STATE $\hat{\mu}_{t}^{+}(i) \gets \frac{m_{t}(i) - n_t(i)}{m_{t}(i)} + \hat{\mu}_{t}^{-}(i)$
                \hfill {\color{gray} \# optimistic estimator for $\mu_i$}
                % {\color{gray} \# optimistic estimator for $\mu_i$}
                \STATE $LCB_{t}(i) \gets \hat{\mu}_{t}^{-}(i)-\sqrt{\frac{2\log T}{m_{t}(i)}}$
                \STATE $UCB_{t}(i) \gets \hat{\mu}_{t}^{+}(i)+\sqrt{\frac{2\log T}{m_{t}(i)}}$
            \ENDFOR
            % \STATE{\color{gray} \# Elimination Step}
            \STATE Remove from $S$ all arms $i$ such that exists $j$ with $UCB_t({i})<LCB_t({j})$
        \ENDWHILE
    \end{algorithmic}
\end{algorithm}

% \tomer{this paragraph feels out of context. is it part of the analysis?}
% Let $\tilde{\mu}_{t}(i)$ be the empirical mean of arm $i$ that is based on all $m_t(i)$ samples. Formally,
% \[
%     \tilde{\mu}_{t}(i) = \frac{1}{m_{t}(i)} \sum_{s < t}\indicator\{a_s=i\}r_{s}.
% \]
% This is the estimator that we would use to compute the confidence interval in non-delayed setting, but since not all observations are available at time $t$, we cannot compute it directly.
% Note that by definition,
% \begin{align} \label{eq:rew-dep-mu-inequality}
%     \forall t,i:
%     \quad 
%     \hat{\mu}_{t}^{-}(i)\leq\tilde{\mu}_{t}(i)\leq\hat{\mu}_{t}^{+}(i)
%     .
% \end{align}
% The above implies that the confidence interval computed by OPSE contain the confidence interval computed by non-delayed SE.

For OPSE we prove the following regret guarantee.
\begin{theorem}
\label{thm:se_reward_dep}
    For reward-dependent delay distributions,
    %$\{\mathcal{D}_i\}_{i=1}^K$, 
    the expected pseudo-regret of \cref{alg:SE-reward-dependent} is bounded by
    \begin{equation}
        \label{eq:se_reward_dep}
        \begin{aligned}
            \mathcal{R}_{T} \leq& \sum_{i \ne i^{*}}\frac{1166\log T}{\Delta_i} 
            + 4\log\left(K\right)\Big(\max_{i\ne i^{*}}d_i(q_i) +d_{i^*}(q_{i^*})\Big),
        \end{aligned}
    \end{equation}
    where $q_{i^*}=1-\min_{i\ne i^{*}}\Delta_i/4$ and $q_i = 1 - \Delta_{i}/4$ for $i\neq i^*$.
\end{theorem}

% \begin{lemma}
% \label{lemma:quantile_bound_hoeffding}
% Fix some $q_i\in\left(0,1\right]$. With probability of at least $1-T^{-4}$, 
% \[
% n_{t}(i)\geq q_i m_{t-d_i(q_i)}(i)-\sqrt{2\log(T) m_{t-d_i(q_i)}(i)}
% \]
% \end{lemma}
% \shahar{TODO: Add paragraph discussing the results. Provide example. Comparing to previous results of SE in independent setting.}
\Cref{thm:se_reward_dep} is analogous to \Cref{thm:SE_arm_quantiles} in the reward-independent setting. We show a variant of SE, rather than PSE, because the algorithm relies on the entire feedback, rather than just the observed feedback. In addition, the dependence in $1/q_i$ was the main motivation to introduce PSE in the previous section, here it is bounded by a constant.
In the reward-dependent setting we have much less information on the unobserved feedback, thus it would be unrealistic to expect similar regret bounds. The main difference between the two bounds is that here we are restricted to specific choice of quantiles $q_i$ and $q_{i^*}$,
%$q_i = 1 - \Delta_{i}/4$ for $i\ne i^*$, and $q_{i^*} = 1 - \min_{i \ne i^*}\Delta_{i}/4$, 
while the bound in \cref{thm:SE_arm_quantiles} hold for any vector $\vec{q}$. 
% while in \cref{thm:SE_arm_quantiles} we have complete control over which quantile we choose. We emphasize that this "choice" is in the analysis of the regret and not actually made by the algorithm.
A second difference between the theorems is in the additive penalty due to the delay, here it is not multiplied by the sub-optimality gap, $\Delta_i$. This factor $\Delta_i$ also appears in the lower bound in \cref{thm:reward_depndent_lower}, which we discuss later on. 

%\yishay{Probably better to replace the proof by a sketch}
\begin{proof}[Proof of \Cref{thm:se_reward_dep} (sketch)] 
    Consider time $t$ in which arm $i$ is still active.
    Define $\lambda_{t}(i) = \sqrt{2\log(T)/m_{t}(i)}$. 
    Let $\tilde{\mu}_{t}(i)$ be the empirical mean of arm $i$ that is based on all $m_t(i)$ samples. Formally, $\tilde{\mu}_{t}(i) = \frac{1}{m_{t}(i)} \sum_{s < t}\indicator\{a_s=i\}r_{s}.$
    
    This is the estimator that we would use to compute the confidence interval in non-delayed setting, but since not all observations are available at time $t$, we cannot compute it directly.
    Note that by definition,
    \begin{align} \label{eq:rew-dep-mu-inequality}
        \forall t,i:
        \quad 
        \hat{\mu}_{t}^{-}(i)\leq\tilde{\mu}_{t}(i)\leq\hat{\mu}_{t}^{+}(i)
        .
    \end{align}
    With high probability, using concentration bound on $\tilde{\mu}_t$ and \cref{eq:rew-dep-mu-inequality} we can show that,
    \begin{equation}
        \label{eq:dep-proof-sketch-delta-bound}
        \begin{aligned}
            \Delta_{i} 
            &=
            \mu_{i^{*}}-\mu_{i}\\
            & \leq
            4\lambda_{t}(i)
            + \hat{\mu}_{t}^{+}(i)-\hat{\mu}_{t}^{-}(i) + \hat{\mu}_{t}^{+}(i^{*})-\hat{\mu}_{t}^{-}(i^{*})
            \\
            & =
            4\lambda_{t}(i)
            +
            \frac{m_{t}(i)-n_{t}(i)}{m_{t}(i)}
            %\underbrace{\frac{m_{t}(i)-n_{t}(i)}{m_{t}(i)}}_{(A)}
            +
            \frac{m_{t}(i^*)-n_{t}(i^*)}{m_{t}(i^*)}
            % \underbrace{\frac{m_{t}(i^*)-n_{t}(i^*)}{m_{t}(i^*)}}_{(B)}.
            .
        \end{aligned}
    \end{equation}
    Let $\dmax = \max_{i \ne i^*}d_i(1-\Delta_{i}/4)$. 
    Using Hoeffding's inequality, with high probability, we have that,
    \begin{align*}
        n_{t}(i)
        \geq
        (1-\Delta_{i}/4) m_{t-\dmax}(i)-\lambda_{t}(i)m_{t}(i).
    \end{align*}
    Hence,
    \begin{align*}
         \frac{m_{t}(i)-n_{t}(i)}{m_{t}(i)} 
        & =
        \frac{m_{t}(i) - m_{t-\dmax}(i)}{m_{t}(i)}
        + 
        \frac{m_{t-\dmax}(i) - n_{t}(i)}{m_{t}(i)}\\
        & \leq
        \frac{m_{t}(i)-m_{t-\dmax}(i)}{m_{t}(i)} + \Delta_{i}/4 + \lambda_{t}(i).
    \end{align*}
    The third term on the right hand side in \cref{eq:dep-proof-sketch-delta-bound} is bounded in a similar fashion, which gives us the following bound:
    \begin{align*}
        \Delta_{i}
        =
        O\Bigg( \frac{2m_{t}(i)-m_{t-\dmax}(i)-m_{t-\dmax^{*}}(i)}{m_{t}(i)}
        %+\frac{m_{t}(i)-m_{t-\dmax^{*}}(i)}{m_{t}(i)}
        + 
        \sqrt{\frac{\log T}{m_{t}(i)}} \Bigg),
    \end{align*}
    where $\dmax^{*} = \max_{i\ne i^*}d_{i^*}(1-\Delta_i)$.
    Either the last term on the right hand side is larger than the first two, or vice versa. By considering both cases and solving them, we yield the following result: 
    \begin{align*}  
        m_{t}(i)\Delta_i
        = O\bigg(     
        \frac{\log T}{\Delta_i} 
        + 
        m_{t}(i)-m_{t-\dmax}(i) 
        +
        m_{t}(i)-m_{t-\dmax^*}(i)
        \bigg).
    \end{align*}
    The above holds for the last time we pull arm $i$, $\tau_{i}$. Summing over the sub-optimal arms gives us a bound on regret. Similar to the setting of \Cref{sec:rew-ind}, $\sum_{i}m_{\tau_{i}}(i) - m_{\tau_{i} - d}(i) \leq \log(K)d$. Here, we set $d$ to $d_{max}$ or $\dmax^{*}$ accordingly, which gives us the desired regret bound.
\end{proof}

\paragraph{Optimistic-UCB.} 

The dependency on the delay of the optimal arm comes from the bias of $\hat{\mu}_{t}^{-}$. A similar proof would hold for a variant of UCB that uses $\hat{\mu}_{t}^{+}$. In that case, one can obtain a regret bound of
\begin{equation}
    \label{eq:regret_UCB_reward_dependent}
     O\left(
    \sum_{i\ne i^{*}}\frac{\log T}{\Delta_{i}}
    + \sum_{i\ne i^{*}}d_{i}(1-\Delta_{i}/4) \right).
\end{equation}
In most cases, this is a weaker bound than the bound of \Cref{thm:se_reward_dep}, as the second term scales linearly with the number of arms. The advantage of Optimistic-UCB is that it does not depend on the delay of the optimal arm. It still remains an open question whether we can enjoy the benefits of both bounds, and achieve a regret bound that depends only on $\max_{i \ne i^*}d_i(1-\Delta_i)$.

On the other hand, in \Cref{thm:reward_depndent_lower} we show that the dependence in $\max_{i \ne i^*}d_i(1-\Delta_i)$ cannot be avoided, which establishes that our bound is not far from being optimal.
\begin{theorem}
    \label{thm:reward_depndent_lower}
    Let $K=2$. For any $\tilde{d} \leq T$ and $\Delta \in [0,1/2] $, there exist reward distributions with sub-optimality gap $\Delta$ and reward-dependent delay distributions with $d_i(1 - 2\Delta) = \tilde{d}$, such that,
    \begin{equation}
        \label{eq:lower_reward_dep}
        \mathcal{R}_T \geq \frac{1}{2}\Delta \cdot d_i(1-2\Delta).
    \end{equation}
    Moreover, for any algorithm $ALG^{delay}$ that guarantees a regret bound of $T^\alpha$ over any instance, the regret is at least,
    \begin{align*}
        \mathcal{R}_T = \Omega \left( \frac{(1-\alpha)\log(T)}{\Delta} + \Delta \cdot d_i(1-2\Delta) \right),
    \end{align*}
    for sufficiently large $T$.
\end{theorem}
Note that $d_i(1 - 2\Delta) \leq d_i(1 - \Delta/4)$, which complies with our upper bound. 
% Furthermore, this opens the question over the factor $\Delta$ omitted in \cref{thm:se_reward_dep}. 
It seems necessary to have the $\Delta$ factor in \cref{eq:lower_reward_dep}, and we conjecture that it should also appear in the upper bound.

The proof for \cref{thm:reward_depndent_lower} is built upon two instances which are indistinguishable until time $\tilde{d}$. The reward distributions are Bernoulli and the index of the optimal arm alternates in the two instances.
% where under both instances $\mu_1 = 1/2$ (this can be known to the learner). 
The idea is that when arm $2$ is optimal, samples with reward $1$ are delayed more often than samples with reward $0$. When arm $2$ is sub-optimal, the opposite occurs. The delay distribution is tailored such that under both instances, \emph{(i)} the probability to observe feedback immediately is exactly $1-2\Delta$; and \emph{(ii)} the probability for reward $1$ given that the delay is $0$, is identical for both arms under both instances. These two properties guarantee that the learner cannot distinguish between the two instances until time $\tilde{d}$. After that, it is possible to distinguish between them whenever a sample with delay $\tilde{d}$ is observed.
The full details of the proof appears in \cref{appendix:dep-lower}.

\section{Experiments}
\label{sec:exp}

% \yishayt{We conducted a variety of synthetic experiments to support our theoretical findings.} 
We conducted a variety of synthetic experiments to support our theoretical findings.

\paragraph{Fixed delays.} 

In \cref{fig:SEvsUSB} we show the effect of different fixed delays on UCB and SE. We ran both algorithms with a confidence radius $\lambda_t(i) = \sqrt{2/n_t(i)}$, for $K = 20$ arms, each with Bernoulli rewards with mean uniform  in $[0.25,0.75]$, under various fixed delays.
%in $[0,500]$. 
Top plots show cumulative regret 
%with differing delays, over time 
until $T=2\cdot 10^4$. Bottom plot shows regret over increasing delays for $T=2\cdot 10^5$. The results are averaged over 100 runs and intervals in both plots are 4 times the standard error.

\begin{figure}
    % \vskip 0.2in
    \centering
    \includegraphics[width=0.75\textwidth]{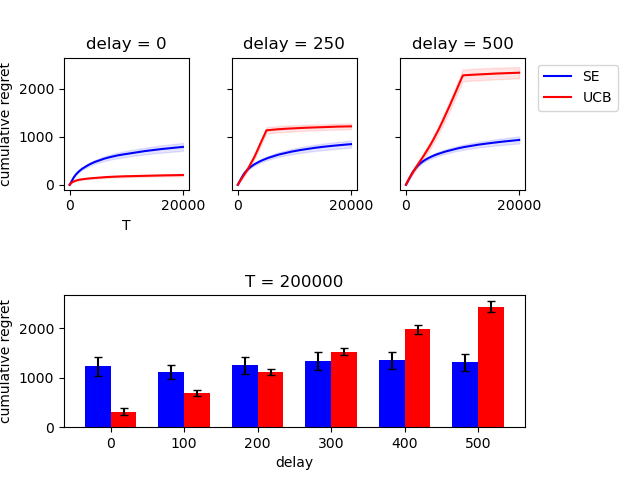}
    % \vspace{-0.75cm}
    \caption{\label{fig:SEvsUSB} Regret of SE and UCB for fixed delays.}
    %in $[0,500]$. $K = 20$ arms, each with Bernoulli rewards with uniform mean in $[0.25,0.75]$. Top plots show cumulative regret with differing delays, over time until $T=2\cdot 10^4$. Bottom plot shows regret over increasing delays on $T=2\cdot 10^5$ horizon. 
    %On the upper plots, the x-axis is time $T = 1,2,3,...,2 \cdot 10^4$ and the y-axis is the cumulative expected regret as a function of $T$. The delay value is denoted on the top. On the lower plot, the x-axis is delay value and the y-axis is the cumulative regret for $T = 2 \cdot 10^5$. 
     
    % and the shaded area around the curves in the top plots, and the black intervals on the lower plot are 4 times the standard error.
    
    % \vskip -0.25in
\end{figure}

%  Under non-delayed setting UCB performs better than SE approximately by a factor of $4$. 
% The gap is due to the nature of SE that needs to shrink the confidence sets both from below and from above (as it involves both LCB and UCB), while UCB needs to shrink it only from above. This approximately translated to a factor $2^2 = 4$ in the regret.
% \shahar{I think this paragraph is too long, commenting it out for now}

As delay increases, the regret of UCB increases as well, while SE is quite robust to the delay, and around delay of 200 SE becomes superior.
% \yishayt{and around delay of 200
% %, the performance switches over and 
% SE becomes superior}. 
These empirical results coincides with our theoretical results: As in the proof \Cref{thm:ucb-lower}, the regret UCB grows linearly in the first $Kd$ rounds. On the other hand, SE created a pipeline of observations, so it keeps getting observations from all active arms. While it cannot avoid from sampling each sub-optimal arm for $d/K$ times, as long as this does not exceed the minimal amount of observations required for SE to eliminate a sub-optimal arm, the effect on the regret is minor.

\paragraph{$\alpha$-Pareto delays.} 

We reproduce an experiment done by \citet{manegueu2020stochastic} under our reward-independent setting, in \cref{fig:SEvsPB}. We compare their algorithm, PatientBandits (PB), with SE. For $T = 3000$ rounds and $K=2$ arms, we ran sub-optimality gaps $\Delta \in [0.04, 0.6]$. The expected rewards are $\mu_1 = 0.4$ and $\mu_2 = 0.4 + \Delta$. The delay is sampled from Pareto distribution with $\alpha_1 = 1$ for arm $1$ and $\alpha_2 = 0.2$ for arm $2$. The results are averaged over 300 runs. 
% \yishay{Both have infinite expected delay, right?!}
% \tal{right}
\begin{figure}
    %\vskip 0.2in
    \centering
    \includegraphics[width=0.75\textwidth]{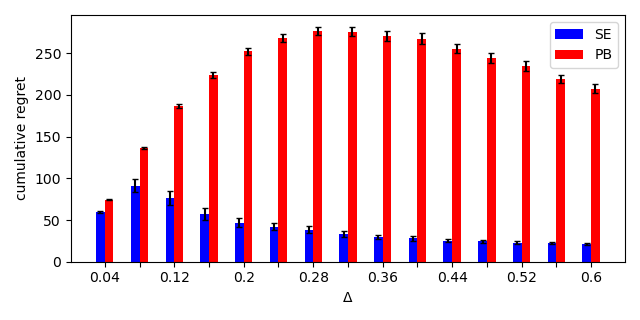}
    % \vspace{-0.75cm}
    \caption{\label{fig:SEvsPB} Regret of SE and PatientBandits (PB) for Pareto delays.}
    %over $T = 3000$ rounds and $K=2$ arms, various values of sub-optimality gap $\Delta \in [0, 0.6]$. The expected rewards are $\mu_1 = 0.4$ and $\mu_2 = 0.4 + \Delta$. The delay is sampled from Pareto distribution with parameter $\alpha_1 = 1$ for arm $1$ and $\alpha_2 = 0.2$ for arm $2$. The results ate averaged over 300 runs. 
    \vskip -0.1in
\end{figure}

PB is a UCB-based algorithm that uses a prior knowledge on distribution in order to tune confidence radius.
Even though it is designed to work under Pareto distributions, 
% and even though it gets the parameters of the distribution as an input, 
SE's regret is strictly smaller for any value of $\Delta$.
For small values of $\Delta$, the regret increases with $\Delta$, as the algorithms are not able to distinct between the arms. When $\Delta$ becomes large enough the regret starts to decrease as $\Delta$ increases. This transition occurs much sooner under SE,
% (approximately when $\Delta = 0.08$ compared to $0.28$ under PB)
which indicates that SE starts to distinguish between the arms at lower values of $\Delta$. We note that PB is designed for partial observation setting, which is more challenging than the reward-independent setting. 
%Having said that, 
However, the work of \cite{manegueu2020stochastic} is the only previous work, as far as know, to present a regret bound for delay distributions that potentially have infinite expected value and arm-dependent delays, as in this experiment.

\paragraph{Packet-loss.}
We study the regret of SE and PSE in the packet loss setting. Specifically to evaluate the difference when amount of feedback from the best arm is significantly smaller than the other arms. We ran the algorithms for $T = 2 \cdot 10^4$ rounds and $K=10$ arms with randomized values of sub-optimality gaps between $\Delta \in [0.15, 0.25]$. The probability to observe the best arm is $0.1$, and $1$ for the sub-optimal arms. The results are averaged over 300 runs.
As seen in \cref{fig:SEvsPSE}, the slope of PSE zeroes in some regions. This is the part of a phase in which the algorithm observed enough feedback from all sub-optimal arms and keeps sampling only the optimal arm. This happens due to the fact that the feedback of the optimal arm is unobserved 90\% of the time. Meanwhile, SE samples each arm equally and receives less reward. The slope of PSE in other regions, is similar to the one of SE which indicates that the set of active arms is similar as well.

\begin{figure}[t]
    %\vskip 0.2in
    \centering
    \includegraphics[width=0.75\textwidth]{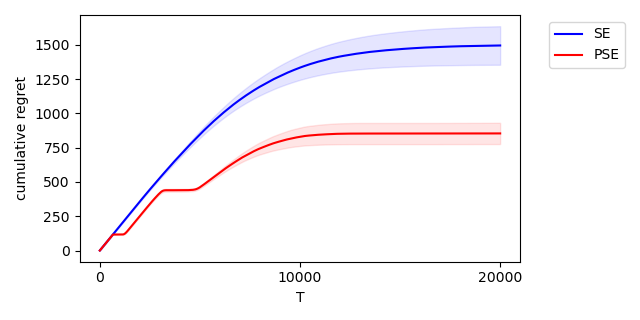}
    \vspace{-0.75cm}
    \caption{\label{fig:SEvsPSE} Regret of SE and PSE for packet loss delays.}  
    
    %\vskip -0.1in
\end{figure}

\paragraph{Reward-dependent case.} 

We compare between OPSE (\cref{alg:SE-reward-dependent}) and UCB. We show that unlike in the reward-independent case, here an "off-the-shelf" solution doesn't perform very well, thus this case requires a modified algorithm. We set $T = 6 \cdot 10^4$  and $K=3$ arms with random sub-optimality gaps of $\Delta \in [0.15, 0.25]$. The delay is \textit{biased} with fixed delay of 5,000 rounds for reward 1 of the best arm and reward 0 of the sub-optimal arms. The results are averaged over 100 runs. In \cref{fig:OPSEvsUCB}, OPSE outperforms UCB, mostly due to UCB's unawareness that the observed reward empirical means are biased. Thus, it favors the sub-optimal arms at the beginning 
%of the run 
and never recovers from that regret loss. We remark that in this settings, standard SE eliminates the best arm and suffers linear regret, so we omitted it from the plot.

\begin{figure}[t]
    %\vskip 0.2in
    \centering
    \includegraphics[width=0.75\textwidth]{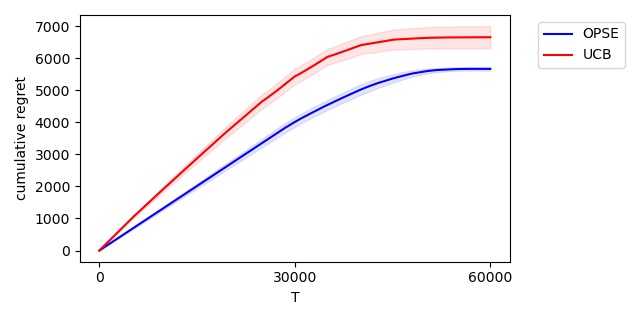}
    \vspace{-0.75cm}
    \caption{\label{fig:OPSEvsUCB} Reward-dependent setting. Regret of OPSE and UCB.}  
    
    % \vskip  -0.1in
    
\end{figure}

We provide additional experiments in \cref{sec:additional-exp}.

 \section{Discussion}
We presented algorithms for multi-arm bandits under two stochastic delayed feedback settings. In the reward-independent, which was studied previously, we present near-optimal regret bounds that scale with the delays quantiles. Those are significantly stronger, in many cases, then previous results. 
% setting The first, the reward-independent delays setting, 
In addition we show a surprising gap between two classic algorithms: UCB and SE. While the former suffers a regret of $\Omega(\mabregret + Kd)$ under fixed delays, the latter achieves $O(\mabregret + d)$ for fixed delays and $O(\min_{q} \mabregret/q + \max_{i}d_i(q))$ in the general setting. We further showed the PSE algorithm, which removes the dependency on the delay of the best arm. We then presented the reward-dependent delay setting, which is more challenging since the observed and the actual rewards distribute differently. Our novel OPSE  algorithm achieves $O(\mabregret + \log(K)d(1-\Delta_{min}))$ by widening the gap of the confidence bounds to incorporate the potential observed biases. In both settings we provided almost matching lower bounds.

%  \textcolor{blue}{We presented algorithms for multi-arm bandits under two stochastic delayed feedback settings. The first, the reward-independent delays setting, has a surprising gap between two classic algorithms: UCB and SE. While the former suffers a regret of $\Omega(\mabregret + Kd)$ under fixed delays, the latter achieves $O(\mabregret + d)$ for fixed delays and $O(\min_{q} \frac{1}{q}\mabregret + \max_{i}d_i(q))$ in the general setting. We further showed the PSE algorithm, which removes the dependency on the delay of the best arm.We then presented the reward-dependent delay setting, which is more challenging since the observed and the actual rewards distribute differently. OPSE, our new algorithm achieves $O(\mabregret + \log(K)d(1-\Delta_{min}))$ by widening the gap of the confidence bounds to incorporate the potential observed biases. In both settings we provided almost matching lower bounds.}
 
Our paper leaves some interesting future lines of research. The reward-dependent setting is mostly unaddressed in the literature and we believe there is more to uncover in this setting. One important question regards the gap between UCB and SE with fixed delays. In non-delayed multi-arm bandits, UCB and SE have similar regret bounds (and UCB even outperforms SE empirically when the delay is zero as evidence by \cref{fig:SEvsUSB}). This raises the question: Can a variant of UCB or any other optimistic algorithm achieve similar regret bounds as a round-robin algorithm in the delayed settings? Lastly, another interesting direction is to tighten the regret bounds: 
In the reward independent case the gap between the lower and upper bound is either logarithmic in $K$ (e.g., the bound in \cref{eq:SE_arms_quantiles}) or missing a $\Delta$  factor on the delay term (e.g., \cref{eq:SE_single_quantile}). 
% We conjecture that the lower bounds are tight and thus a new algorithm is required to match it from above.
In the reward dependent case it is still remains open question whether we can enjoy the benefits of both optimistic-SE and optimistic-UCB and obtain a regret bound that scales with $\max_{i\ne i^*}d_i(1-\Delta_i)$.

%  \tomer{this is too generic. we need to provide a little bit more details. otherwise remove.}
%, including comparison between SE and PSE in the Packet Loss settings.

%\todo{Change arvix versions to articles}
% \section{Discussion}    \label{sec:discussion}

%\nocite{langley00}

\section*{Acknowledgments}
The work of YM and TL has received funding from the European Research Council (ERC) under the European Union's Horizon 2020 research and innovation program (grant agreement No. 882396), by the Israel Science Foundation (grant number 993/17) and the Yandex Initiative for Machine Learning at Tel Aviv University. 
SS and TK were supported in part by the Israeli Science Foundation (ISF) grant no.~2549/19, by the Len Blavatnik and the Blavatnik Family foundation, and by the Yandex Initiative in Machine Learning.

%  \tomer{this is too generic. we need to provide a little bit more details. otherwise remove.}
%, including comparison between SE and PSE in the Packet Loss settings.

%\todo{Change arvix versions to articles}
% \section{Discussion}    \label{sec:discussion}

%\nocite{langley00}
\newpage
\bibliography{refs}

\begin{thebibliography}{21}
\providecommand{\natexlab}[1]{#1}
\providecommand{\url}[1]{\texttt{#1}}
\expandafter\ifx\csname urlstyle\endcsname\relax
  \providecommand{\doi}[1]{doi: #1}\else
  \providecommand{\doi}{doi: \begingroup \urlstyle{rm}\Url}\fi

\bibitem[Auer and Ortner(2010)]{auer2010ucb}
P.~Auer and R.~Ortner.
\newblock Ucb revisited: Improved regret bounds for the stochastic multi-armed
  bandit problem.
\newblock \emph{Periodica Mathematica Hungarica}, 61\penalty0 (1-2):\penalty0
  55--65, 2010.

\bibitem[Auer et~al.(2002{\natexlab{a}})Auer, Cesa-Bianchi, and
  Fischer]{auer2002finite}
P.~Auer, N.~Cesa-Bianchi, and P.~Fischer.
\newblock Finite-time analysis of the multiarmed bandit problem.
\newblock \emph{Machine learning}, 47\penalty0 (2-3):\penalty0 235--256,
  2002{\natexlab{a}}.

\bibitem[Auer et~al.(2002{\natexlab{b}})Auer, Cesa-Bianchi, Freund, and
  Schapire]{auer2002nonstochastic}
P.~Auer, N.~Cesa-Bianchi, Y.~Freund, and R.~E. Schapire.
\newblock The nonstochastic multiarmed bandit problem.
\newblock \emph{SIAM journal on computing}, 32\penalty0 (1):\penalty0 48--77,
  2002{\natexlab{b}}.

\bibitem[Bistritz et~al.(2019)Bistritz, Zhou, Chen, Bambos, and
  Blanchet]{bistritz2019online}
I.~Bistritz, Z.~Zhou, X.~Chen, N.~Bambos, and J.~Blanchet.
\newblock Online exp3 learning in adversarial bandits with delayed feedback.
\newblock In \emph{Advances in Neural Information Processing Systems}, pages
  11349--11358, 2019.

\bibitem[Cesa-Bianchi et~al.(2018)Cesa-Bianchi, Gentile, and
  Mansour]{cesa2018nonstochastic}
N.~Cesa-Bianchi, C.~Gentile, and Y.~Mansour.
\newblock Nonstochastic bandits with composite anonymous feedback.
\newblock In \emph{Conference On Learning Theory}, pages 750--773, 2018.

\bibitem[Cesa-Bianchi et~al.(2019)Cesa-Bianchi, Gentile, and
  Mansour]{cesa2019delay}
N.~Cesa-Bianchi, C.~Gentile, and Y.~Mansour.
\newblock Delay and cooperation in nonstochastic bandits.
\newblock \emph{The Journal of Machine Learning Research}, 20\penalty0
  (1):\penalty0 613--650, 2019.

\bibitem[Csisz{\'a}r and Talata(2006)]{csiszar2006context}
I.~Csisz{\'a}r and Z.~Talata.
\newblock Context tree estimation for not necessarily finite memory processes,
  via bic and mdl.
\newblock \emph{IEEE Transactions on Information theory}, 52\penalty0
  (3):\penalty0 1007--1016, 2006.

\bibitem[Desautels et~al.(2014)Desautels, Krause, and
  Burdick]{desautels2014parallelizing}
T.~Desautels, A.~Krause, and J.~W. Burdick.
\newblock Parallelizing exploration-exploitation tradeoffs in gaussian process
  bandit optimization.
\newblock \emph{Journal of Machine Learning Research}, 15:\penalty0 3873--3923,
  2014.

\bibitem[Dudik et~al.(2011)Dudik, Hsu, Kale, Karampatziakis, Langford, Reyzin,
  and Zhang]{dudik2011efficient}
M.~Dudik, D.~Hsu, S.~Kale, N.~Karampatziakis, J.~Langford, L.~Reyzin, and
  T.~Zhang.
\newblock Efficient optimal learning for contextual bandits.
\newblock In \emph{Proceedings of the Twenty-Seventh Conference on Uncertainty
  in Artificial Intelligence}, pages 169--178, 2011.

\bibitem[Even-Dar et~al.(2006)Even-Dar, Mannor, and Mansour]{even2006action}
E.~Even-Dar, S.~Mannor, and Y.~Mansour.
\newblock Action elimination and stopping conditions for the multi-armed bandit
  and reinforcement learning problems.
\newblock \emph{Journal of machine learning research}, 7\penalty0
  (Jun):\penalty0 1079--1105, 2006.

\bibitem[Gael et~al.(2020)Gael, Vernade, Carpentier, and
  Valko]{manegueu2020stochastic}
M.~A. Gael, C.~Vernade, A.~Carpentier, and M.~Valko.
\newblock Stochastic bandits with arm-dependent delays.
\newblock In \emph{International Conference on Machine Learning}, pages
  3348--3356. PMLR, 2020.

\bibitem[Gy{\"o}rgy and Joulani(2020)]{gyorgy2020adapting}
A.~Gy{\"o}rgy and P.~Joulani.
\newblock Adapting to delays and data in adversarial multi-armed bandits.
\newblock \emph{arXiv preprint arXiv:2010.06022}, 2020.

\bibitem[Joulani et~al.(2013)Joulani, Gyorgy, and
  Szepesv{\'a}ri]{joulani2013online}
P.~Joulani, A.~Gyorgy, and C.~Szepesv{\'a}ri.
\newblock Online learning under delayed feedback.
\newblock In \emph{International Conference on Machine Learning}, pages
  1453--1461, 2013.

\bibitem[Kleinberg et~al.(2010)Kleinberg, Niculescu-Mizil, and
  Sharma]{kleinberg2010regret}
R.~Kleinberg, A.~Niculescu-Mizil, and Y.~Sharma.
\newblock Regret bounds for sleeping experts and bandits.
\newblock \emph{Machine learning}, 80\penalty0 (2-3):\penalty0 245--272, 2010.

\bibitem[Lattimore and Szepesv{\'a}ri(2020)]{lattimore2020bandit}
T.~Lattimore and C.~Szepesv{\'a}ri.
\newblock \emph{Bandit algorithms}.
\newblock Cambridge University Press, 2020.

\bibitem[Pike-Burke et~al.(2018)Pike-Burke, Agrawal, Szepesvari, and
  Grunewalder]{pike2018bandits}
C.~Pike-Burke, S.~Agrawal, C.~Szepesvari, and S.~Grunewalder.
\newblock Bandits with delayed, aggregated anonymous feedback.
\newblock In \emph{International Conference on Machine Learning}, pages
  4105--4113. PMLR, 2018.

\bibitem[Thune et~al.(2019)Thune, Cesa-Bianchi, and
  Seldin]{thune2019nonstochastic}
T.~S. Thune, N.~Cesa-Bianchi, and Y.~Seldin.
\newblock Nonstochastic multiarmed bandits with unrestricted delays.
\newblock In \emph{Advances in Neural Information Processing Systems}, pages
  6541--6550, 2019.

\bibitem[Vernade et~al.(2017)Vernade, Capp{\'e}, and
  Perchet]{vernade2017stochastic}
C.~Vernade, O.~Capp{\'e}, and V.~Perchet.
\newblock Stochastic bandit models for delayed conversions.
\newblock In \emph{Conference on Uncertainty in Artificial Intelligence}, 2017.

\bibitem[Weinberger and Ordentlich(2002)]{weinberger2002delayed}
M.~J. Weinberger and E.~Ordentlich.
\newblock On delayed prediction of individual sequences.
\newblock \emph{IEEE Transactions on Information Theory}, 48\penalty0
  (7):\penalty0 1959--1976, 2002.

\bibitem[Zhou et~al.(2019)Zhou, Xu, and Blanchet]{zhou2019learning}
Z.~Zhou, R.~Xu, and J.~Blanchet.
\newblock Learning in generalized linear contextual bandits with stochastic
  delays.
\newblock In \emph{Advances in Neural Information Processing Systems}, pages
  5197--5208, 2019.

\bibitem[Zimmert and Seldin(2020)]{zimmert2020optimal}
J.~Zimmert and Y.~Seldin.
\newblock An optimal algorithm for adversarial bandits with arbitrary delays.
\newblock In \emph{International Conference on Artificial Intelligence and
  Statistics}, pages 3285--3294. PMLR, 2020.

\end{thebibliography}


\begin{thebibliography}{8}
\providecommand{\natexlab}[1]{#1}
\providecommand{\url}[1]{\texttt{#1}}
\expandafter\ifx\csname urlstyle\endcsname\relax
  \providecommand{\doi}[1]{doi: #1}\else
  \providecommand{\doi}{doi: \begingroup \urlstyle{rm}\Url}\fi

\bibitem[Author(2021)]{anonymous}
Author, N.~N.
\newblock Suppressed for anonymity, 2021.

\bibitem[Duda et~al.(2000)Duda, Hart, and Stork]{DudaHart2nd}
Duda, R.~O., Hart, P.~E., and Stork, D.~G.
\newblock \emph{Pattern Classification}.
\newblock John Wiley and Sons, 2nd edition, 2000.

\bibitem[Kearns(1989)]{kearns89}
Kearns, M.~J.
\newblock \emph{Computational Complexity of Machine Learning}.
\newblock PhD thesis, Department of Computer Science, Harvard University, 1989.

\bibitem[Langley(2000)]{langley00}
Langley, P.
\newblock Crafting papers on machine learning.
\newblock In Langley, P. (ed.), \emph{Proceedings of the 17th International
  Conference on Machine Learning (ICML 2000)}, pp.\  1207--1216, Stanford, CA,
  2000. Morgan Kaufmann.

\bibitem[Michalski et~al.(1983)Michalski, Carbonell, and
  Mitchell]{MachineLearningI}
Michalski, R.~S., Carbonell, J.~G., and Mitchell, T.~M. (eds.).
\newblock \emph{Machine Learning: An Artificial Intelligence Approach, Vol. I}.
\newblock Tioga, Palo Alto, CA, 1983.

\bibitem[Mitchell(1980)]{mitchell80}
Mitchell, T.~M.
\newblock The need for biases in learning generalizations.
\newblock Technical report, Computer Science Department, Rutgers University,
  New Brunswick, MA, 1980.

\bibitem[Newell \& Rosenbloom(1981)Newell and Rosenbloom]{Newell81}
Newell, A. and Rosenbloom, P.~S.
\newblock Mechanisms of skill acquisition and the law of practice.
\newblock In Anderson, J.~R. (ed.), \emph{Cognitive Skills and Their
  Acquisition}, chapter~1, pp.\  1--51. Lawrence Erlbaum Associates, Inc.,
  Hillsdale, NJ, 1981.

\bibitem[Samuel(1959)]{Samuel59}
Samuel, A.~L.
\newblock Some studies in machine learning using the game of checkers.
\newblock \emph{IBM Journal of Research and Development}, 3\penalty0
  (3):\penalty0 211--229, 1959.

\end{thebibliography}

\appendix

\newpage
\section{Reward-independent Setting}

We state a simple concentration bound for the estimation of the observed expected rewards, which follows immediately from Hoeffding's inequality and a union bound. 

\begin{lemma}
    \label{lemma:estimator_bound}
    Let $\hat{\mu}_t(i)$ be the observed empirical average of the expected reward up to the end of round $t-1$. Then,
    \begin{align*}
        \Pr \Big[ \exists ~ i, t \;:\; \abs{\hat{\mu}_t(i) - \mu_i} > \sqrt{\frac{2\log{T}}{n_t(i)}}\Big] 
        &\leq 
        \frac{2}{T^2}.
    \end{align*}
\end{lemma}

\subsection{Proof of \Cref{thm:ucb-lower}}
\label{appendix:ucb-lower}

    Consider an instance in which all rewards are sampled from Bernoulli distributions, where the optimal arm has $\mu_{i^{*}} = 1$ and the rest of the arms are with mean $\mu_i=\frac{1}{2}$ and the delay is fixed such that $K < d$. Without loss of generality we assume that the tie breaking rule is by index and that $i^* \geq K/2$. We can always randomize the index of the optimal arm. In this case, the assumption holds with probability of at least $1/2$, which affects the regret only by a constant. Recall that UCB begins with round-robin over all arms. Let $r_i$ be the first realization of reward from arm $i$ (at time $t=i$). With a constant probability the reward of at least $1/4$ of arms $i\leq i^*$ is $1$. 
    Formally, using Chernoff concentration bound,
    \[
        \Pr\left( \sum_{i\leq i^{*}}r_{i}\geq\frac{K}{4} \right)
        \geq
        1-e^{-\frac{1}{8}K}
        \geq
        \frac{1}{10}.
    \]
    This means, that when calculating the UCB with respect to at most one sample, with a constant probability, there are at least $K/4$ arms $i\leq i^*$, that has higher UCB than the optimal arm. Let these arms with lowest index be $i_1<...<i_{K/4}$. Additionally, assume $i_1 = 1$, which occurs with probability $1/2$. 
    
    Since $K < d$, until time $K+d$ we either do not observe any feedback (until time $d+1$) or we observe some of the feedback from the round-robin (from time $d+1$ to $K+d$). As arm $1$'s first reward is $1$, it is the only arm we sample until time $d+K$, since it has the lowest index and maximal UCB for that period. At time $K+d+1$, we observe a second feedback for arm $1$, which lowers arm 1's UCB. At that time, arm $i_2$ has the lowest index with maximal UCB. We then sample it $d$ times until time $K+2d$, as no new observation is coming from any arm other than arm $1$ (which has lower UCB). At time $K+2d+1$ we observe a second feedback for arm $i_2$, and we move to sample $i_3$. This process is repeated and we sample $d$ times each of arms $i_2, ..., i_{K/4}$ consecutively. Therefore, the total regret of UCB under this instance scale as,
    \[
        \mathcal{R}_{T} \geq \frac{1}{20}\frac{Kd}{4}\Delta \geq \Omega(Kd).
    \]
    \begin{remark}
        The proof relies on an assumption that the first arm that we pull is also the first arm in the tie-breaking rule. That way, even if a feedback with reward of $1$ is observed within the interval $[d+1,d+K]$, we keep sampling that arm. This assumption can be easily avoided, simply by multiplying the reward of the rest of the arm by $1-\epsilon$ (for sufficiently small $\epsilon > 0$), so that even if a positive feedback is received, the UCB of the first arm we pulled is still the highest.
    \end{remark}

\subsection{Proof of \Cref{thm:SE_arm_quantiles}}
\label{appendix:proof_SE_arm_quantiles}

\begin{lemma}
    \label{lemma:quantile_bound_chernoff}
    At time $t$, for any arm $i$ and quantile $q \in (0,1]$, it holds that,
    $$
        \Pr \Big[n_{t +d_i(q)}(i) < \frac{q}{2}  m_t(i) \Big] \leq \exp \Big(-\frac{q}{8} m_t(i) \Big).
    $$
\end{lemma}
\begin{proof}[Proof of \Cref{lemma:quantile_bound_chernoff}]
    Define $\indicator \{d_s \leq d_i(q)\}$ to be an indicator that on time $s$ that the delay is smaller than $d_i(q)$. Note that if arm $i$ was pulled at time $s$, then $\E[\indicator \{d_s \leq d_i(q)\} \mid a_s = i] \geq q$. Thus,
    \begin{align*}
        \Pr \Big[n_{t +d_i(q)}(i) < \frac{q}{2}  m_t(i) \Big] &\leq \Pr \Big[\sum_{s\leq t: a_s = i} \indicator \{d_s \leq d_i(q)\}  < \frac{q}{2}  m_t(i) \Big]
        \\
        & \leq \Pr \Big[\sum_{s\leq t: a_s = i} \indicator \{d_s \leq d_i(q)\}  < \frac{1}{2}  \sum_{s\leq t: a_s = i} \E[\indicator \{d_s \leq d_i(q)\} \mid a_s = i] \Big]
        \\
        & \leq \exp \Big(-\frac{1}{8} \sum_{s\leq t: a_s = i} \E[\indicator \{d_s \leq d_i(q)\} \mid a_s = i] \Big) \leq \exp \Big(-\frac{q}{8} m_t(i) \Big),
    \end{align*}
    where the third inequality follows from the relative Chernoff bound, and the last inequality is since $\sum_{s\leq t: a_s = i} \E[\indicator \{d_s \leq d_i(q)\} \mid a_s = i] \geq q \cdot m_t(i)$.
\end{proof}

    \paragraph{Regret bound \cref{eq:SE_arms_quantiles}:}
    Fix some vector $\vec{q}\in(0,1]^K$, and define $d_{max}(i) = \max\{d_i(q_i),d_{i^{*}}(q_{i^{*}})\}$. Consider the following failure events:
    \begin{align*}
        F_1
        & =
        \left\{\exists t,i:
            \abs{\hat{\mu}_{t}(i)-\mu_i} > \sqrt{\frac{2\log(T)}{n_{t}(i)}}
        \right\}, \\
        F_2 
        & = 
        \left\{ 
            \exists t,i:\;m_{t}(i)\geq\frac{32\log(T)}{q_i}~,~n_{t+d_{max}({i})}(i)<\frac{q_i}{2}m_{t}(i)
        \right\}.
    \end{align*}
    
    Since the delays are independent of the rewards, the reward
    estimator is unbiased. By \cref{lemma:estimator_bound}, $\Pr\left( F_1 \right) \leq 2 T^{-2}$.
    % Let $\sigma\in S_{K}$ be a permutation on $[K]$, such that $\sigma(i)$ represent the number of arms remaining after $i$ was eliminated (e.g. if $\sigma(i)=K-1$ then $i$ is the first arm to be eliminated), where we break ties arbitrarily.
    % \tal{Do we really need the $\sigma$ from this point..? I think it would be more readable if we mention it only in (*) below. And use $n_{t+d_{max}(i)+K}$ instead}
    Using union bound and \cref{lemma:quantile_bound_chernoff},
    \begin{align*}
        \Pr[F_2] & =\Pr\Big[\exists t, i :~ m_{t}(i)\geq\frac{32\log(T)}{q_i} ~,~ n_{t+d_{max}(i)}(i) < \frac{q_i}{2}m_{t}(i)\Big]\\
         & \tag{union bound}
         \leq
         \sum_{i}\sum_{t:m_{t}(i)\geq32\log(T)/q_i} \Pr\Big[n_{t+d_{max}(i)}(i)<\frac{q_i}{2}m_{t}(i)\Big]\\
         & \leq
         \sum_{i}\sum_{t:m_{t}(i)\geq32\log(T)/q_i} \Pr\Big[n_{t+d_i(q_i)}(i)<\frac{q_i}{2}m_{t}(i)\Big]\\
         & \leq
         \tag{by \cref{lemma:quantile_bound_chernoff}}
         \sum_{i}\sum_{t:m_{t}(i)\geq32\log(T)/q_i}\exp\Big(-\frac{q_i}{8}m_{t}(i)\Big)\\
         & \leq T\cdot K\exp\Big(-\frac{q_i}{8}\frac{32\log(T)}{q_i}\Big)\leq\frac{1}{T^{2}}.
    \end{align*}
    Define the good event $G = \lnot F_1 \cap \lnot F_2$. By union bound, $\Pr(G) \geq 1 - 3T^{-2}$.
    The good event (particularly, $\lnot F_1$) implies that,
    \[
        UCB_t(i^*) \geq \mu_{i^*} \geq \mu_i \geq LCB_t(i),
    \]
    for any $i$. Hence, the best arm is never eliminated. The event $\lnot F_2$ implies a lower bound on the number of observations received from each
    arm.  
    We bound the regret under $\lnot G$ by $T$, and for the rest of the analysis we assume that $G$ occurred.
    % The probability of $\mathcal{F}$ is negligible, thus we'll assume
    % this event does not occur.

    Let $\tau_i$ be the last time we performed elimination and arm $i$ was not eliminated in it (meaning it remained active for exactly one more round-robin and $m_T(i) \leq m_{\tau_i}(i) + 1$). By the algorithm's definition, 
    %\tal{change $\tau_i$ to $\tau_i$ throughout the paper}
    \[
         LCB_{\tau_i}({i^{*}})\leq UCB_{\tau_i}(i).
    \]
    The above with $\lnot F_1$ implies that,
    
    \begin{align}   
        \Delta_i = \mu_{i^{*}} - \mu_i 
        \leq 
        2\sqrt{\frac{2\log(T)}{n_{\tau_i}(i)}} + 2\sqrt{\frac{2\log(T)}{n_{\tau_i}(i^*)}}.
        \label{eq:SE-proof-first-delta-bound}
    \end{align}
    Assume that 
    \begin{align}   \label{eq:SE-proof-enogh-samples}
        m_{\tau_i-d_{max}(i)}(i) > 32\log(T) \max\{ \frac{1}{q_{i}} , \frac{1}{q_{i^*}} \}.
    \end{align}
     Since $F_2$ does not occur, 
     \begin{align*}
         n_{\tau_i}(i)
         & \geq
         \frac{q_i}{2}m_{\tau_i-d_{max}(i) }(i),\\
         n_{\tau_i}({i^{*}})
         & \geq
         \frac{q_{i^*}}{2}m_{\tau_i-d_{max}(i) }(i^*) \\
         & \geq
         \tag{$\forall$ active arms $i,j$: $|m_t(i) - m_t(j)| \leq 1$}
         \frac{q_{i^*}}{2}(m_{\tau_i-d_{max}(i)})(i) - 1).\\
     \end{align*}
     
    Combining with \cref{eq:SE-proof-first-delta-bound} gives us,
    \begin{align*}
        \Delta_i 
        & \leq 
        2\sqrt{\frac{2\log(T)}{q_i (m_{\tau_i-d_{max}(i) }(i) - 1)}}
        + 2\sqrt{\frac{2\log(T)}{q_{i^{*}}(m_{\tau_i-d_{max}(i) }(i) - 1)}}
        \\
        & \leq
        6\sqrt{\frac{\log(T)}{m_{\tau_i-d_{max}(i)}(i) - 1} 
        \left(  \frac{1}{q_{i^{*}}} + \frac{1}{q_{i}} \right) },
    \end{align*}
    where the last inequality uses $\sqrt{a}+\sqrt{b}\leq2\sqrt{a+b}$. Now, this implies that,
    
    $$
        m_{\tau_i-d_{max}(i) }(i) 
        \leq 
        \frac{36\log(T)}{(\Delta_i)^{2}}\left(\frac{1}{q_i}+\frac{1}{q_{i^{*}}}\right) + 1
        \leq
        \frac{37\log(T)}{(\Delta_i)^{2}}\left(\frac{1}{q_i}+\frac{1}{q_{i^{*}}}\right).
    $$
    If the condition in \cref{eq:SE-proof-enogh-samples} does not hold, then the above holds trivially. The regret of arm $i$ is therefore,
    \begin{align*}
        m_{T}(i)\Delta_i 
        & \leq 
        (1 + m_{\tau_i}(i))\Delta_i\\  
        & \leq 
        \Delta_i + (m_{\tau_i-d_{max}(i)}(i)+m_{\tau_i}(i)-m_{\tau_i-d_{max}(i)}(i))\Delta_i\\
         & \leq 
         \Delta_i  + \frac{37\log(T)}{\Delta_i}\left(\frac{1}{q_i}+\frac{1}{q_{i^{*}}}\right)+(m_{\tau_i}(i)-m_{\tau_i-d_{max}(i)}(i))\Delta_i
         \\
         & \leq 
         \frac{38\log(T)}{\Delta_i}\left(\frac{1}{q_i}+\frac{1}{q_{i^{*}}}\right)+(m_{\tau_i}(i)-m_{\tau_i-d_{max}(i)}(i))\Delta_i.
    \end{align*}
    Let $\sigma(i)$ be number of active arms at time $\tau_i$. Since we round-robin over active arms,
    \[
        m_{\tau_i}(i)-m_{\tau_i-d_{max}(i)}(i)
        \leq
        \frac{d_{max}(i)}{\sigma(i)} + 1.
    \]
    So the total regret can be bounded as,
    \begin{align*}
        \mathcal{R}_{T} 
        & \leq 
        \sum_{i\ne i^{*}} \frac{38\log(T)}{\Delta_i}\left(\frac{1}{q_i}+\frac{1}{q_{i^{*}}}\right) + \sum_{i\ne i^{*}} \mathbb{E}\left[\frac{d_{max}(i)}{\sigma(i)}\Delta_i + 1 \right] 
        + \underbrace{\Pr(\lnot G) \cdot T}_{\leq 1}
        \\
        &  \leq \sum_{i\ne i^{*}}\frac{40\log(T)}{\Delta_i}\left(\frac{1}{q_i}+\frac{1}{q_{i^{*}}}\right) + \sum_{i\ne i^{*}} \mathbb{E}\left[\frac{d_i(q_i) + d_{i^*}(q_{i^*})}{\sigma(i)}\Delta_i\right]
        \\
         &\leq \sum_{i\ne i^{*}}\frac{40\log(T)}{\Delta_i}\left(\frac{1}{q_i}+\frac{1}{q_{i^{*}}}\right) + \max_{i\neq i^*} (d_i(q_i) + d_{i^*}(q_{i^*}))\Delta_i \Big(\sum_{i\ne i^{*}}  \mathbb{E}\left[\frac{1}{\sigma(i)}\right]\Big)
        \\
         & \leq \sum_{i\ne i^{*}}\frac{40\log(T)}{\Delta_i}\left(\frac{1}{q_i}+\frac{1}{q_{i^{*}}}\right) + \log(K)\max_{i \neq i^*}(d_i(q_i)+d_{i^*}(q_{i^*}))\Delta_i),
    \end{align*}
    where the last inequality follows $\sum_{i \ne i^{*}} 1/\sigma(i) \leq 1/K + 1/(K-1) + ... + 1/2 \leq \log K$.
    The above holds for any choice of $\vec{q}$, and in particular hold for the $\vec{q}$ that minimizes the bound.

% \subsection{Proof of \Cref{thm:SE_single_quantile}}
% \label{appendix:reward-independent-single-q}
\paragraph{Regret bound of \cref{eq:SE_single_quantile}:}
Fix $q\in(0,1]$. Let $t_{\ell}$ be the time we pulled all the active arms exactly $32\log T/(q\epsilon_{\ell}^{2}$) times, where $\epsilon_{\ell} = 2^{-\ell}$.
We define the next two failure events:
\[
    F_{1} = 
    \left\{ \exists t,i: \abs{\hat{\mu}_{t}(i) - \mu_{t}(i)}>\sqrt{\frac{2\log(T)}{n_{t}(i)}} \right\},
\]
\[
    F_{2} = 
    \left\{ \exists t,i:
    m_t(i) \geq \frac{32\log(T)}{q}, n_{t+d_{max}}(i) < \frac{1}{2}qm_{t}(i) \right\} ,
\]
where $d_{max} = \max_{i}d_i(q)$.
Define the clean event $G=\lnot F_{1} \cap \lnot F_{2}$. Using \cref{lemma:estimator_bound}, \cref{lemma:quantile_bound_chernoff} and union bound, $\Pr(G)\geq1-2T^{-2}$. Recall that under event $G$, $i^{*}$ is never eliminated. 

Let $0 \leq \kappa_{\ell} \leq K$ such that $t_{\ell} + d_{max} + \kappa_{\ell}$ is an elimination step. Let $S_{\ell}$ be the set of sub-optimal arms, that where \textit{not} eliminated by time $t_{\ell} + d_{max} + \kappa_{\ell}$, but was eliminated by time $t_{\ell+1} + d_{max} +  \kappa_{\ell+1}$. If arm $i$ was never eliminated, we consider $i$ to be part of $S_{\ell}$ for the minimal $\ell$ such that $t_{\ell+1}+d_{max} +  K > T$. We  define $\kappa_{\ell}$ as such, since $t_{\ell} + d_{max}$ is not necessarily a time where SE performs an elimination and could be in the middle of a round-robin.

Given $i \in S_{\ell}$, since $i$ was not eliminated in that step, 
\[
    UCB_{t_{\ell} + d_{max} + \kappa_{\ell}}(i)
    \geq 
    LCB_{t_{\ell} + d_{max} + \kappa_{\ell}}(i^{*}).
\]
The above implies that 
\begin{align}
    \Delta_i 
    = 
    \mu_{i^{*}} - \mu_{i} 
    \leq 
    2\sqrt{\frac{2\log(T)}{n_{t_{\ell} + d_{max} + \kappa_{\ell}}(i)}} + 2\sqrt{\frac{2\log(T)}{n_{t_{\ell}+  d_{max} + \kappa_{\ell}}(i^{*})}}.
    \label{eq:fixed-q-delta-bound1}
\end{align}
Under the event $G$, 
\begin{align*}
    n_{t_{\ell} + d_{max} + \kappa_{\ell}}(i) 
        & \geq \frac{q}{2}m_{t_{\ell}}(i)
        \underset{(*)}{\geq} 16\log(T)/\epsilon_{\ell}^{2},\\
    n_{t_{\ell} + d_{max} + \kappa_{\ell}}(i^*) 
        & \geq \frac{q}{2}m_{t_{\ell}}(i^*)
        \underset{(*)}{\geq} 16\log(T)/\epsilon_{\ell}^{2},
\end{align*}
where $(*)$ is definition of $t_\ell$. Combing with \cref{eq:fixed-q-delta-bound1} gives us,
\begin{equation}    \label{eq:fixed-q-delta-bound2}
    \Delta_{i} 
    \leq
    \frac{3}{2}\epsilon_{\ell}
    =
    3\epsilon_{\ell+1}.
\end{equation}
Hence, the total regret from the arms in $S_{\ell}$ is,
\begin{align*}
    \sum_{i\in S_{\ell}}m_{t_{\ell + 1} + d_{max} + \kappa_{\ell+1}}(i)\Delta_{i} 
    & = \sum_{i\in S_{\ell}} \big(m_{t_{\ell + 1} + d_{max} + \kappa_{\ell+1}}(i)-m_{t_{\ell + 1}}(i)\big)\Delta_{i} + \sum_{i\in S_{\ell}}m_{t_{\ell + 1}}(i)\Delta_{i}
    \\
    & \leq 3\sum_{i\in S_{\ell}} \big(m_{t_{\ell + 1} + d_{max} + \kappa_{\ell+1}}(i)-m_{t_{\ell + 1}}(i)\big) \epsilon_{\ell + 1} 
    + \sum_{i\in S_{\ell}}\frac{32\log(T)}{q\epsilon_{\ell + 1}^{2}}\Delta_{i}
    \\
    & \leq 3\sum_{i\in S_{\ell}} \big(m_{t_{\ell + 1} + d_{max} + \kappa_{\ell+1}}(i)-m_{t_{\ell + 1}}(i)\big) \epsilon_{\ell + 1}
    + \sum_{i\in S_{\ell}}\frac{288\log(T)}{q\Delta_{i}}
    \\
    & \leq 3\sum_{i\in [K]} \big(m_{t_{\ell + 1} + d_{max} + \kappa_{\ell+1}}(i)-m_{t_{\ell + 1}}(i)\big) \epsilon_{\ell + 1}
    + \sum_{i\in S_{\ell}}\frac{288\log(T)}{q\Delta_{i}}
    \\
    & \leq 3(d_{max} + K) \epsilon_{\ell + 1}
    + \sum_{i\in S_{\ell}}\frac{288\log(T)}{q\Delta_{i}}.
\end{align*}
The first two inequalities are by \cref{eq:fixed-q-delta-bound2}, the third inequality is since $m_t(i)$ is increasing with $t$, and the last equality is since $\sum_i m_t(i) = t$ for any $t$. 

Summing over all $\ell \in \N$, taking into account the arms that where
eliminated before time $t_{0}+d_{max}+\kappa_0$ and the bad event,
\begin{align*}
    \mathcal{R}_{T} 
    & \leq
    \sum_{t=1}^{t_{0}+d_{max}+\kappa_0}\sum_i \indicator\{a_t = i\}\Delta_{i} + 3(d_{max} + K) \underbrace{\sum_{\ell = 0}^{\infty} \epsilon_{\ell + 1}}_{ = 1}
    +\sum_{\ell = 0}^{\infty} \sum_{i\in S_{\ell}} \frac{288\log(T)}{q\Delta_{i}} + \Pr(\lnot G)T\\
    & \leq
    \frac{32\log(T)}{q}(K-1) + 4(d_{max} + K)
    +\sum_{i\ne i^*} \frac{289\log(T)}{q\Delta_{i}}\\
    & \leq
    \sum_{i\ne i^*} \frac{325\log(T)}{q\Delta_{i}} + 4\max_{i}d_{i}(q).
\end{align*}
% \[
% \mathcal{R}_{T}
% \leq \sum_{\ell=0}^{\infty} 2\sqrt{\epsilon_{\ell+1}} \max_{i \in S_{\ell}} \Big((d_{i}(q) + K)\sqrt{\Delta_{i}}\Big) + \sum_{i\in S_{\ell}}\frac{512\log(T)}{q\Delta_{i}} \leq \sum_{i \in S_{\ell}} \frac{512\log(T)}{q\Delta_{i}} + 10d_{i}(q) + 10K + \frac{2}{T}
% \]
The above is true for any non-zero choice of $q$, thus we choose
the optimal $q$ to obtain the statement of the theorem.

%The regret bound in \cref{eq:SE_fixed_delay} follows immediately from setting $q = 1$ in \cref{eq:SE_single_quantile}.

\subsection{Proof of \Cref{thm:PSE_reward_indep}}
\label{appendix:PSE-proof}
    Fix some vector $\vec{q}\in(0,1]^K$, and define the following failure events:
    \begin{align*}
        F_1
        & =
        \left\{\exists t,i:
            \abs{\hat{\mu}_{t}(i)-\mu_i} > \sqrt{\frac{2\log(T)}{n_{t}(i)}}
        \right\}, \\
        F_2 
        & = 
        \left\{ 
            \exists t,i: m_{t}(i)\geq\frac{32\log(T)}{q_i}~,~n_{t+d_{i}(q_i)}(i)<\frac{q_i}{2}m_{t}(i)
        \right\},
    \end{align*}
    and the good event $G = \lnot F_1 \cap \lnot F_2$. Similar to the proof of \Cref{thm:SE_arm_quantiles}, %\Cref{appendix:proof_SE_arm_quantiles}, 
    $\Pr(G) \geq 1 - 3T^{-2}$.
    
    Let $t_{\ell}$ be the last round of phase $\ell$. 
    Assume arm $i$ is eliminated at time $t_{\ell + 1}$. Since it is not eliminated at time $t_{\ell}$,
    \[
        LCB_{t_{\ell}}(i^*) \leq UCB_{t_{\ell}}(i),
    \]
    which under the good event implies,
    \begin{align}
        \Delta_{i}
        =
        \mu_{i^{*}}-\mu_{i}
        \leq 2\sqrt{\frac{2\log(T)}{n_{t_{\ell}}(i)}} + 2\sqrt{\frac{2\log(T)}{n_{t_{\ell}}(i^*)}}  \leq 4\sqrt{\frac{2\log(T)}{16\log(T)/\epsilon_{\ell}^2}} 
        \leq 
        \sqrt{2}\epsilon_{\ell} 
        = 
        2\sqrt{2}\epsilon_{\ell+1}.
        \label{eq:PSE-proof-delta-bound}
    \end{align}
    where the second inequality is because the phase ends whenever $n_{t_{\ell}}(j) \geq 16\log(T)/\epsilon_{\ell}^2$ for all $j$.
    Let $\tau_i$ the last time we've pulled arm $i$. We have that $n_{\tau_i - 1}(i) < 16\log(T)/\epsilon_{\ell + 1}^2$.
    Assume that,
    \begin{align}   \label{eq:PSE-proof-enogh-samples}
        m_{\tau_i-d_i(q_i) - 1}(i) >  \frac{32\log(T)}{q_{i}}.
    \end{align}
    Under the good event,
    \begin{align*}
        m_{\tau_i-d_i(q_i) - 1}(i) 
        & \leq 
        \frac{2}{q_i}n_{\tau_i - 1}(i)
        \leq
        \frac{32\log(T)}{q_{i}\epsilon_{\ell + 1}^2}
        \leq
        \frac{256\log(T)}{q_{i}\Delta_{i}^{2}}
    \end{align*}
    where the last holds by \cref{eq:PSE-proof-delta-bound}. If the condition in \cref{eq:PSE-proof-enogh-samples} does not holds, then the above holds trivially.
    The total regret from arm $i$ is,
    \begin{align*}
        m_{\tau_i}(i)\Delta_{i} 
        & =
        m_{\tau_i-d_{i}(q_{i})-1}(i)\Delta_{i}+(m_{\tau_i}(i)-m_{\tau_i-d_{i}(q_{i})-1}(i))\Delta_{i}\\
        & \leq
        \frac{256\log(T)}{q_{i}\Delta_{i}}+(m_{\tau_i}(i)-m_{\tau_i-d_{i}(q_{i})-1}(i))\Delta_{i}\\
        % & =
        % \frac{288\log(T)}{q_{i}\Delta_{i}}+\sum_{t=\tau_i-d_{i}(q_{i})}^{\tau_i}(m_{t}(i)-m_{t-1}(i))\Delta_{i}\\
        & =
        \frac{256\log(T)}{q_{i}\Delta_{i}}+\sum_{t=\tau_i-d_{i}(q_{i})}^{\tau_i}\indicator\{a_{t}=i\}\Delta_{i}
    \end{align*}
    When summing over all arms,
    \begin{align}
        \nonumber
        \mathcal{R}_{T} 
        & \leq
        \sum_{i\ne i^{*}} \frac{256\log(T)}{q_{i}\Delta_{i}}  +  \sum_{i\ne i^{*}}\sum_{t=\tau_i-d_{i}(q_{i})}^{\tau_i} \indicator\{a_{t}=i\}\Delta_{i} + \Pr(\lnot G)T\\
        \nonumber
        & \leq
        \sum_{i\ne i^{*}}\frac{257\log(T)}{q_{i}\Delta_{i}} + \sum_{i\ne i^{*}}\sum_{t=\tau_i-d_{i}(q_{i})}^{\tau_i} \indicator\{a_{t}=i\}\Delta_{i}\\
        & =
        \sum_{i\ne i^{*}} \frac{257\log(T)}{q_{i}\Delta_{i}}
        + \sum_{\ell=1}^{L} \sum_{i = 1 }^{K} \sum_{t=1}^{T} \indicator\{a_{t}=i,t\in[\tau_i-d_{i}(q_{i}),\tau_i] \cap [t_{\ell-1} + 1,t_{\ell}]\}\Delta_{i},
        \label{eq:PSE-proof-regret}
    \end{align}
    where $L$ is number of phases and we define $t_0 = 0$. Let $S_{\ell}$ be the set of all arms, $i$, such that some rounds in $[\tau_i - d_i(q_i),\tau_i]$ intersects with phase $\ell$. Formally, 
    %\shahar{$t_{\ell}$ is the end of round $\ell$, indices are off.}
    \[
        S_{\ell}=
        \bigl\{ i\in[K]:[\tau_i - d_i(q_i),\tau_i]\cap[t_{\ell-1} + 1,t_{\ell}] \ne \emptyset
        \bigr\}.
    \]
    Let $\sigma_{\ell}(i)$ be the number of arms in the round-robin of phase $\ell$ at time $\min\{ \tau_i , t_{\ell}\}$. We have that, 
    %\tal{throughout the appendix, don't break the first equality, if possible}
    \begin{align}
        \nonumber
        & \sum_{i=1}^{K} \sum_{t=1}^{T} \indicator\{a_{t}=i,t\in[\tau_i-d_{i}(q_{i}),\tau_i]\cap[t_{\ell-1} + 1,t_{\ell}]\}\Delta_{i}\\
        \nonumber
        & =
        \sum_{i\in S_{\ell}} \sum_{t=1}^{T}
        \indicator\{a_{t}=i, t\in[\tau_i-d_{i}(q_{i}),\tau_i] \cap [t_{\ell-1} + 1,t_{\ell}]\}\Delta_{i}\\
        \nonumber
        & \leq
        \sum_{i\in S_{\ell}}\left[\frac{(d_{i}(q_{i})+1)\Delta_{i}}{\sigma_{\ell}(i)} + 1\right]\\
        \nonumber
        & \leq
        \sum_{i\in S_{\ell}} \frac{\max_{i\ne i^{*}}(d_{i}(q_{i})+1)\Delta_{i}}{\sigma_{\ell}(i)} + |S_\ell|\\
        & \leq(\log(K) + 1)\max_{i\ne i^{*}}d_{i}(q_{i})\Delta_{i}+\log(K) + K.
        \label{eq:PSE-proof-penalty-per-phase}
    \end{align}
    The equality is since the sum over the indicators is empty whenever $i \notin S_{\ell}$. The first inequality follows the fact that $|[\tau_i-d_{i}(q_{i}),\tau_i]| \leq d_i(q_i) + 1$ and that we round-robin over at least $\sigma_\ell(i)$ arms. Finally, the last inequality is since 
    \[
        \sum_{i\in S_{\ell}} \sigma_{\ell}(i)
        \leq
        \frac{1}{|S_{\ell}|} + \frac{1}{|S_{\ell}|-1} + ... + 1
        \leq
        \log(|S_{\ell}|) + 1
        \leq
        \log(K) + 1.
    \]
    Plugging \cref{eq:PSE-proof-penalty-per-phase} in \cref{eq:PSE-proof-regret}, and using the fact that the number of phases is at most $\log(T)$, gives us,
    \[
        \mathcal{R}_T
        \leq
        \sum_{i\ne i^{*}} \frac{290\log(T)}{q_{i}\Delta_{i}}
        + \log(T) (\log(K) + 1) \max_{i\ne i^{*}}d_{i}(q_{i})\Delta_{i}.
    \]

\subsection{Lower Bound and proof of \cref{thm:lower-bound-ind}}    \label{appendix:lower-bound-ind}

% \yishay{Where is the proof of Theorem 4 from the main text? Corollary 2 comes the closest but it is different!}

We will use the following lower bound for MAB without delay, which is a variant of \citep[Lemma 14]{kleinberg2010regret}
\begin{lemma}   \label{lemma:no-delay-lower-bound}
    Consider an algorithm $ALG^{MAB}$ for MAB problem without delays. And let $\Iber$ be the set of instances with Bernoulli rewards. Fix  sub-optimality gaps $\Delta_{i}\in [0,1/4]$ where $\Delta_{1} = 0$ (that is, arm $1$ is optimal), and consider the instance $I \in \Iber$ in which $\mu_i = 1/2 - \Delta_{i}$.
    If $ALG^{MAB}$'s regret over all instances in $\Iber$, is bounded by $CT^{\alpha}$ where $\alpha \in [0,1]$ and $C > 0$, then,
    \[
        \mathcal{R}_{T}^{ALG}
        \geq
        \frac{1}{32}\sum_{i:\Delta_{i}>0}
        \left(
            \frac{(1 - \alpha)\log T-\log\frac{8C}{\Delta_{i}}}{\Delta_{i}}
        \right)^{+}
    \]
    where $(x)^{+} = \max\{x , 0\}$.
\end{lemma}

\begin{proof}
    % Consider the the following subset of $\Iber$:
    % \[
    %     \Ical = \left\{ I \in \Iber : \mu_{i}(I) \in \{ 1/2 - \Delta_{i},  1/2 + \Delta_{i}\} \right\},
    % \]
    % where $\mu_{i}(I)$ denotes the expected reward of arm $i$ under instance $I$.
    We denote by $\mu_{i}(I')$, the expected reward of arm $i$ under instance $I'$. For any $i \ne i^*$ consider the instance $I_{i}$ under which
    \[
        \mu_{j}(I_{i})
        = 
        \begin{cases}
            \frac{1}{2} - \Delta_{j} & j \ne i\\
            \frac{1}{2} + \Delta_{i} & j=i
        \end{cases}
    \]
    Instance $I_i$ is similar to instnace $I$. However, under instance $I_i$, arm $i$ is the optimal arm. The assumption that $ALG^{MAB}$'s regret over any instance, and in particular over $I_i$, is small, would insure that $ALG^{MAB}$ cannot specialized on $I$.
    Formally, $ALG^{MAB}$'s regret over $I_{i}$ is at most $CT^{\alpha}$ and so, by \citep[Lemma 16.3]{lattimore2020bandit},
    \[
        \E_{I}\left[m_{i}(T)\right]
        \geq
        \frac{(1-\alpha)\log(T)-\log\left(\frac{8C}{\Delta_{i}}\right)}
        {KL\left(\mu_{i}(I)\Vert\mu_{i}(I_{i})\right)}
        \geq
        \frac{(1-\alpha)\log(T)-\log\left(\frac{8C}{\Delta_{i}}\right)}{32\Delta_{i}^{2}},
    \]
    where $KL(p \Vert q)$ is KL-divergence between Bernoulli distributions with parameters $p$ and $q$. The seconed inequality is due to inverse Pinsker's inequality (see for example \citep[Lemma 6.3]{csiszar2006context}). The lemma now follows by multiplying the above by $\Delta_i$ (to get a bounds on the regret from arm $i$), and summing over all arms.
\end{proof}

\newpage
\begin{theorem}     
    \label{thm:lower-bound-ind-0-inf}
    Consider a delay distribution such that with probability $q$ the delay is $0$, and infinity otherwise. For any  sub-optimality gaps $\Delta_{i}\in [0,1/4]^K$ where $\Delta_{1} = 0$ (that is, arm $1$ is optimal), consider an instance $I$ in which $\mu_i = 1/2 - \Delta_{i}$.
    %\tal{fix sentence}
    For any algorithm $ALG^{delay}$ that guarantees a regret bound of $CT^\alpha$ over any instance, $ALG^{delay}$'s regret on $I$ is at least,
    \begin{align*}
        \Rcal_{T}^{ALG^{delay}}
        & \geq
        \frac{1}{32}\sum_{i:\Delta_{i}>0}\frac{(1 - \alpha)\log T - 2\log\left(\frac{32C}{q\Delta_{i}}\right) - 4}{q\Delta_{i}}
    \end{align*}
    
\end{theorem}

\begin{proof}
    Let $ALG^{delay}$ be an algorithm that guarantees expected pseudo regret of $CT^{\alpha}$, over $T$ rounds, for any instance. We built an algorithm $ALG^{MAB}$ that simulates $ALG^{delay}$ and interacts with a non-delayed environment for $\frac{1}{4} q T$ rounds. In each round, $ALG^{MAB}$ draw a Bernoulli variable with probability $q$. If the variable is $1$, then it chooses the same action as $ALG^{delay}$ and feed it with the feedback. Otherwise, only $ALG^{delay}$ plays this round. If after $\lfloor T(1-q/4)\rfloor$ rounds of $ALG^{delay}$, $ALG^{MAB}$ have not played $(Tq)/4$ rounds, then for the rest of its rounds it follows $ALG^{delay}$s actions (and keeps feeding it with feedback with probability $q$). This technical condition  ensures that $ALG^{MAB}$ plays all of the $(Tq)/4$ rounds. However with high probability $ALG^{MAB}$ finishes his game before that. A protocol of this process appears as \Cref{alg:lower-bound-ind}. 
    \begin{algorithm}
        \caption{Reduction from Non-Delayed Environment}    \label{alg:lower-bound-ind}
        \begin{algorithmic}
            \FOR{$t = 1,2,3,..., \lfloor T(1-q/4)\rfloor$}
                \STATE $ALG^{delay}$ chooses an action $a_t$. 
                \STATE $ALG^{MAB}$ draw $X_t \sim Ber(q)$.
                \IF{$X_t = 1$}
                    \STATE $ALG^{MAB}$ plays $a_t$ and feed $ALG^{delay}$ with the feedback $r_t(a_t)$.                
                \ENDIF
            \ENDFOR
            \STATE {\color{gray} \# The condition below occurs under ``bad'' event}
            \IF{$\sum_{t=1}^{\lfloor T(1-q/4)\rfloor}X_{t} < \frac{qT}{4}$}
                \STATE $ALG^{MAB}$ follows $ALG^{delay}$s choices for the rest of its game, while keeps feeding $ALG^{MAB}$ with the feedbacks with probability of $q$.
            \ENDIF
        \end{algorithmic}
    \end{algorithm}
    
    Define the bad event as
    \[
        B = \left\{  \sum_{t=1}^{\lfloor T(1-q/4)\rfloor}X_{t} < \frac{qT}{4} \right\}.
    \]
    Since $\lfloor T(1-q/4)\rfloor \geq T/2$, by Chernoff bound,
    \[
        \Pr(B)
        \leq
        e^{-Tq/16}
    \]
    Let $C' = \frac{C}{\left(q/4\right)^{\alpha}}$. We have that, 
    \[
        \Rcal^{ALG^{MAB}}_{\lfloor \frac{1}{4}Tq \rfloor} 
        \leq 
        \Rcal^{ALG^{delay}}_{T}
        \leq
        C T^{\alpha} = C'\left\lfloor \frac{Tq}{4}\right\rfloor ^{\alpha}.
    \]
    By \Cref{lemma:no-delay-lower-bound}, 
    \begin{align}
        \Rcal^{ALG^{MAB}}_{\lfloor \frac{1}{4}Tq \rfloor} 
        \geq
        \frac{1}{32}\sum_{i:\Delta_{i}>0}
        \left(
            \frac{(1 - \alpha)\log \frac{Tq}{4} + \log\frac{\Delta_{i}}{8C'}}{\Delta_{i}}
        \right)^{+}
        =
        \frac{1}{32}\sum_{i:\Delta_{i}>0}
        \left(
            \frac{(1-\alpha) \log T + \log\left(\left(\frac{q}{4}\right)^{1+\alpha}\frac{\Delta_{i}}{8C}\right)}
            {\Delta_{i}}
        \right)^{+}
    \label{eq:ind-lower-bound-ALG-tag-regret}
    \end{align}
    On the other hand,
    \begin{align*}
        \Rcal_{\lfloor\frac{1}{4}Tq\rfloor}^{ALG^{MAB}} 
        & \leq
        \E\left[\sum_{t<T(1-q/4)}\indicator\{X_{t}=1\}\sum_{i=1}^{K}\indicator\{a_{t}=i\}\Delta_{i}\right]
        +\E\left[\indicator\{B\}\sum_{t\geq T(1-q/4)}\sum_{i=1}^{K}\indicator\{a_{t}=i\}\Delta_{i}\right]\\
        \tag{$X_{t}$ is independent of $a_{t}$} 
        & \leq
        \sum_{t<T(1-q/4)}
        \underbrace{\E\left[\indicator\{X_{t}=1\}\right]}_{=q}
        \E\left[\sum_{i=1}^{K}\indicator\{a_{t}=i\}\Delta_{i}\right]+\frac{Tq}{4}\Pr(B)\\
        & \leq
        q\E\left[\sum_{t=1}^{T}\sum_{i=1}^{K}\indicator\{a_{t}=i\}\Delta_{i}\right]+\frac{Tq}{4}\Pr(B)\\
        & \leq 
        q\Rcal_{T}^{ALG^{delay}} + \frac{Tq}{4}e^{-\frac{Tq}{16}}\\
        & \leq 
        q\Rcal_{T}^{ALG^{delay}} + 4.
    \end{align*}
    Combing with \cref{eq:ind-lower-bound-ALG-tag-regret},
    \begin{align*}
        \mathcal{R}_{T}^{ALG^{delay}} 
        & \geq
        \frac{1}{32}\sum_{i:\Delta_{i}>0}\frac{(1-\alpha)\log T + \log\left(\left(\frac{q}{4}\right)^{1+\alpha}\frac{\Delta_{i}}{8C}\right)}{q\Delta_{i}} - \frac{4}{q}\\
        & \geq
        \frac{1}{32}\sum_{i:\Delta_{i}>0}\frac{(1 - \alpha)\log T - 2\log\left(\frac{32C}{q\Delta_{i}}\right) - 4}{q\Delta_{i}}
    \end{align*}
\end{proof}

Fix some integer $\tilde{d}$ and consider the following delay distribution:
\begin{align}
    \forall i\in[K]:\quad\Pr(d(i)=x)
    =
    \begin{cases}
        q   & x = \Tilde{d}\\
        1-q & x = \infty.
    \end{cases}
    \label{eq:lower-bound-indpendent-delay-d-inf}
\end{align}
Under the distribution above, the algorithm does not get feedback in the first $\Tilde{d}$ rounds. Therefore, if we fix sub-optimality gaps and arrange the arms in uniform order, then the expected regret of any algorithm will be at least $\Tilde{d}\Deltabar$ where $\Deltabar = \frac{1}{K}\sum_{i}\Delta_i$. Combining this observation with \Cref{thm:lower-bound-ind-0-inf} gives us the next corollary.
\begin{corollary}
   For any sub-optimality gaps set $S_\Delta = \{\Delta_i : \Delta_i \in [0,1/4]\}$ of cardinality $K$, a quantile $q\in(0,1]$, and $\Tilde{d} \leq T$, there  exist a delay distribution, and instance $I$  with an order on $S_\Delta$, such that for any algorithm $ALG^{delay}$ that guarantees a regret bound of $CT^\alpha$ over any instance, $ALG^{delay}$s regret on $I$ is at least,
    \[
        \mathcal{R}_T
        \geq
        \frac{1}{64}\sum_{i:\Delta_{i}>0}\frac{(1 - \alpha)\log T - 2\log\left(\frac{32C}{q\Delta_{i}}\right) - 4}{q\Delta_{i}}
        + \frac{1}{2}\Deltabar\max_i d_i(q).
    \]
    Moreover, $d_i(q) = \Tilde{d}$ for any $i$.
\end{corollary}
\Cref{thm:lower-bound-ind} is obtained directly from the above, as for $T \geq \exp\left({\frac{8}{(1-\alpha)} (\log \left( \frac{32C}{q\min\Delta_i} \right) + 1 )}\right)$,
\[
    \forall i \ne i^*: \quad
    \frac{1}{2}(1-\alpha)\log T
    \geq
    2\log\left(\frac{32C}{q\Delta_{i}}\right) + 4.
\]

\newpage
\section{Reward-dependent Setting} 
\label{appendix:reward-dependent}
\subsection{Proof of \Cref{thm:se_reward_dep}}
We begin with proving 3 useful lemmas.
The first, it a concentration bound for the estimation of the actual empirical expected rewards, similar to \cref{lemma:estimator_bound}, which also follows immediately from Hoeffding's inequality and a union bound. 

\begin{lemma}
    \label{lemma:estimator_bound_unobserved}
    Let $\tilde{\mu}_t(i)$ be the actual empirical average (including unobserved feedback) of the expected reward up to the end of round $t-1$. Then,
    \begin{align*}
        \Pr \Big[ \exists ~ i, t \;:\; \abs{\tilde{\mu}_t(i) - \mu_i} > \sqrt{\frac{2\log{T}}{m_t(i)}}\Big] 
        &\leq 
        \frac{2}{T^2}.
    \end{align*}
\end{lemma}

    \begin{lemma}
    \label{lemma:quantile_bound_hoeffding}
    Fix some $q \in (0,1]$. For any $i$ and $t$, with probability of at least $1-T^{-4}$, 
    \begin{align}
        n_{t}(i) \geq q m_{t-d_i(q)}(i) - \sqrt{2\log(T) m_{t}(i)}.
        \label{eq:rew-dep-good}
    \end{align}
\end{lemma}
    \begin{proof}
    By definition, $\Pr[d_{s}\leq d_i(q)|a_{s}=i] \geq q$. 
    Hence, by Hoeffding's inequality 
    \[
        \Pr\Big[\frac{1}{m_{t-d_i(q)}(i)}\sum_{s=1}^{t-d_i(q)}\indicator\left\{ d_{s}\leq d_i(q),a_{s}=i\right\} \leq q-\delta\Big]\leq \exp(-2m_{t-d_i(q)}(i)\delta^{2}).
    \]
    For $\delta=\sqrt{2\log(T)/m_{t-d_i(q)}(i)}$ 
    \[
        \Pr\Bigg[\sum_{s=1}^{t-d_i(q)}\indicator\left\{ d_{s}\leq d_i(q),a_{s}=i\right\} \leq qm_{t-d_i(q)}(i)-\sqrt{2\log(T)m_{t-d_i(q)}(i)}\Bigg]\leq\frac{1}{T^{4}}.
    \]
    Now, note that
    \[
        n_{t}(i)=\sum_{s=1}^{t}\indicator\left\{ s+d_{s}\leq t,a_{s}=i\right\} \geq\sum_{s=1}^{t-d_i(q)}\indicator\left\{ d_{s}\leq d_i(q),a_{s}=i\right\},
    \]
    which implies that with probability of at least $1-T^{-4}$, 
    \[
    n_{t}(i)
    \geq qm_{t-d_i(q)}(i)-\sqrt{2\log(T) m_{t-d_i(q)}(i)}
    \geq qm_{t-d_i(q)}(i)-\sqrt{2\log(T) m_{t}(i)},
    \]
    where the last inequality is since $m_t(i)$ is monotone in $t$.
\end{proof}
    
    \begin{lemma}   
        \label{lemma:reward-dependent-interval-size-bound}
        Fix some time $t$, arm $i$, and $q\in (0,1]$. If \cref{eq:rew-dep-good} holds, then,
        \[
            \hat{\mu}_{t}^{+}(i)-\hat{\mu}_{t}^{-}(i)\leq\frac{m_{t}(i)-m_{t-d_i(q)}(i)}{m_{t}(i)}+1-q+\sqrt{\frac{2\log(T)}{m_{t}(i)}}.
        \]
    \end{lemma}
    \begin{proof}
        Using the condition in \cref{eq:rew-dep-good},
        \begin{align*}
            \hat{\mu}_{t}^{+}(i)-\hat{\mu}_{t}^{-}(i)
            & = 
            \frac{m_{t}(i)-n_{t}(i)}{m_{t}(i)}
            \\
             & =\frac{m_{t}(i)-m_{t-d_i(q)}(i)+m_{t-d_i(q)}(i)-n_{t}(i)}{m_{t}(i)}
             \\
             & \leq \frac{m_{t}(i)-m_{t-d_i(q)}(i)}{m_{t}(i)}
             + \frac{m_{t-d_i(q)}(i)(1-q)+\sqrt{2\log(T)m_{t}(i)}}{m_{t}(i)}
             \\
             & \leq\frac{m_{t}(i)-m_{t-d_i(q)}(i)}{m_{t}(i)}+1-q+\sqrt{\frac{2\log(T)}{m_{t}(i)}},
        \end{align*}
        where the second inequality holds since $m_{t-d_i(q)}(i)\leq m_{t}(i)$.
    \end{proof}

    We now turn to prove the theorem. Fix some vector $\vec{q} \in (0,1]^K$, which will be determined later. Define the following failure events:
    \begin{align*}
        F_1
        & =
        \left\{\exists t,i:
            \abs{\tilde{\mu}_{t}(i)-\mu_i} > 
            \lambda_t(i)
        \right\}, \\
        F_2 
        & = 
        \left\{ 
            \exists t,i: n_{t}(i) < q_i m_{t-d_i(q_i)}(i)-\lambda_t(i) m_t(i)
        \right\},\\
        F_3 
        & = 
        \left\{ 
            \exists t,i: n_{t}(i^*) < q_i m_{t-d_{i^*}(q_i)}(i)-\lambda_t(i) m_t(i^*)
        \right\},
    \end{align*}
    where $\lambda_t(i) = \sqrt{2\log T/m_t(i)}$. By \Cref{lemma:estimator_bound}, \Cref{lemma:quantile_bound_hoeffding}, and the union bound, each of the event above occures with probability of at most $T^{-2}$. Define the good event $G = \lnot F_1 \cup \lnot F_2 \cup \lnot F_3$. The probability that $G$ occurs is at least $1 - 3T^{-2}$, by the union bound. We bound the regret under $\lnot G$ by $T$, and for the rest of analysis we assume that $G$ occurs. 
    
    Recall that $\tilde{\mu}_{t}(i)$ is the empirical mean of arm $i$ that is based on all $m_t(i)$ samples. Formally, $\tilde{\mu}_{t}(i) = \frac{1}{m_{t}(i)} \sum_{s < t}\indicator\{a_s=i\}r_{s}.$ By definition,
    \begin{align} \label{eq:rew_dep_mu_inequality_appendix}
        \forall t,i:
        \quad 
        \hat{\mu}_{t}^{-}(i)\leq\tilde{\mu}_{t}(i)\leq\hat{\mu}_{t}^{+}(i)
        .
    \end{align}

    Using \cref{eq:rew_dep_mu_inequality_appendix}, under the good event. For any $i$,
    \begin{align*}
        \mu_{i}
        & \leq
        \tilde{\mu}_{t}(i) + \lambda_{t}(i)
        \leq
        \hat{\mu}_{t}^{+}(i) + \lambda_{t}(i)
        = UCB_{t}(i),\\
        \mu_{i}
        & \geq
        \tilde{\mu}_{t}(i) - \lambda_{t}(i)
        \geq
        \hat{\mu}_{t}^{-}(i) - \lambda_{t}(i)
        = LCB_{t}(i).
    \end{align*}
    The above implies that for any $i$,
    \[
        UCB_{t}(i^{*}) 
        \geq \mu_{i^{*}} 
        \geq \mu_{i} 
        \geq LCB_{t}(i).
    \]
    Therefore, the optimal arm is never eliminated. Let $\tau_i$ be the last elimination step in which $i$ was not eliminated. So the total number of times that we have pulled $i$ is at most $m_{\tau_i}(i)+1$. Since $i$ was not eliminated,
    \begin{align*}
        \hat{\mu}_{\tau_i}^{+}(i) + \lambda_{\tau_i}(i)
        & =
        UCB_{\tau_i}(i)
         \geq 
        LCB_{\tau_i}(i^{*})
        =
        \hat{\mu}_{\tau_i}^{-}(i^{*}) - \lambda_{\tau_i}(i^{*}).
    \end{align*}
    We can bound the right hand-side from above by,
    \begin{align*}
        \hat{\mu}_{\tau_i}^{-}(i^{*}) - \lambda_{\tau_i}(i^{*}) 
        & =
        \tilde{\mu}_{\tau_i}(i^{*}) - \lambda_{\tau_i}(i^{*}) + \hat{\mu}_{\tau_i}^{-}(i^{*}) - \tilde{\mu}_{\tau_i}(i^{*})\\
        & \geq
        \mu_{i^{*}} - 2\lambda_{\tau_i}(i^{*}) + \hat{\mu}_{\tau_i}^{-}(i^{*}) - \hat{\mu}_{\tau_i}^{+}(i^{*}),
    \end{align*}
    where the last follows by the good event and \cref{eq:rew_dep_mu_inequality_appendix}. Similarly,
    \begin{align*}
        \hat{\mu}_{\tau_i}^{+}(i) - \lambda_{\tau_i}(i) 
        & \leq
        \mu_{i} + 2\lambda_{\tau_i}(i) + \hat{\mu}_{\tau_i}^{+}(i) - \hat{\mu}_{\tau_i}^{-}(i).
    \end{align*}
    Combining last three inequalities, and the fact that, $\lambda_{\tau_i}(i^*) = \lambda_{\tau_i}(i)$ (since $\tau_i$ is an elimination step),
    \begin{equation}
        \label{eq:dep-proof-delta-bound}
        \begin{aligned}
            & \Delta_{i} 
            =
            \mu_{i^{*}}-\mu_{i}
             \leq
            4\lambda_{\tau_i}(i)
            +
            \hat{\mu}_{\tau_i}^{+}(i)-\hat{\mu}_{\tau_i}^{-}(i)
            +
            \hat{\mu}_{\tau_i}^{+}(i^{*})-\hat{\mu}_{\tau_i}^{-}(i^{*}).
        \end{aligned}
    \end{equation}
    Using \Cref{lemma:reward-dependent-interval-size-bound},
    \begin{align*}
        \hat{\mu}_{\tau_i}^{+}(i)-\hat{\mu}_{\tau_i}^{-}(i)
        & \leq
        \frac{m_{\tau_i}(i) - m_{\tau_i - d_i(q_i)}(i)}{m_{\tau_i}(i)}+ 1-q_i + \lambda_{\tau_i}(i),\\
        \hat{\mu}_{\tau_i}^{+}(i^{*})-\hat{\mu}_{\tau_i}^{-}(i^{*})
        & \leq
        \frac{m_{\tau_i}(i) - m_{\tau_i-d_{i^*}(q_{i})}(i) + 1}{m_{\tau_i}(i)} + 1-q_{i} + \lambda_{\tau_i}(i),
    \end{align*}
    where for the second inequality we have also used the fact that $m_{\tau_i}(i) = m_{\tau_i}(j)$ and $|m_{t}(i) - m_{t}(j)|\leq 1$, for any active arms $i,j$ and time $t$. Plugging in \cref{eq:dep-proof-delta-bound}, setting $q_i = 1 - \Delta_{i}/4$, and using the fact that $m_t$ is monotonically increasing in $t$, 
    % \shahar{Maybe expand into multiple steps}
    \begin{align*}
        \Delta_{i}
        & \leq 2\frac{m_{\tau_i}(i)-m_{\tau_i-d_{max}}(i)}{m_{\tau_i}(i)}
         + 
        2\frac{m_{\tau_i}(i)-m_{\tau_i-d_{max}^{*}}(i)+1}{m_{\tau_i}(i)}
        + 
        12\lambda_{\tau_i}(i),
    \end{align*}
    where $d_{max} = \max_{i\ne i^{*}}d_{i}(1-\Delta_{i}/4)$ and $d_{max}^{*} = \max_{i\ne i^{*}}d_{i^{*}}(1-\Delta_{i}/4)$. Now, if the last term on the right-hand-side dominates the other two, then, 
    \[
        \Delta_{i}\leq24\sqrt{\frac{2\log T}{m_{\tau_{i}}(i)}}
        \Longrightarrow 
        m_{\tau_{i}}(i)\Delta_{i}
        \leq 1152\frac{\log T}{\Delta_{i}}.
    \]
    Otherwise,
    \begin{align*}
        m_{\tau_{i}}(i)\Delta_{i} 
        & \leq
        4 \left(m_{\tau_{i}}(i)-m_{\tau_{i}-d_{max}}(i) \right)
         +
        4 \left(m_{\tau_{i}}(i)-m_{\tau_{i}-d_{max}^{*}}(i) \right) + 4.
    \end{align*}
    Either way,
    \begin{equation}
        \begin{aligned}
            m_{T}(i)\Delta_{i} 
            & \leq
            \left( m_{\tau_{i}}(i) + 1 \right)\Delta_{i}\\
            & \leq
            1157\frac{\log T}{\Delta_{i}} 
            +
            4 \left(m_{\tau_{i}}(i)-m_{\tau_{i}-d_{max}}(i) \right)
             +
            4 \left(m_{\tau_{i}}(i)-m_{\tau_{i}-d_{max}^{*}}(i) \right).
        \end{aligned}
        \label{eq:rew-dep-proof-arm-regret}
    \end{equation}
    Let $\sigma(i)$ be the number of active arms at time $\tau_i$. Since we round-robin over the arms,
    \[
        m_{\tau_{i}-\tilde{d}}(i)-m_{\tau_{i}}(i)\leq\frac{\tilde{d}}{\sigma(i)}+1,
    \]
    for any integer $\tilde{d}$. Note that, $\sum_{i\ne i^*} 1/\sigma(i) \leq 1/K + 1/(K-1) + ... +1/2 \leq \log(K)$. Summing \cref{eq:rew-dep-proof-arm-regret} over the sub-optimal arms, and taking into account the bad event.
    \begin{align*}
        \mathcal{R}_{T} 
        & \leq
        \sum_{i\ne i^{*}}\frac{1157\log T}{\Delta_{i}} + 4\log(K)(d_{max}+d_{max}^{*})
         +
        8(K-1) 
        + \underbrace{\Pr(\lnot G)T}_{\leq 1}\\
        & \leq
        \sum_{i\ne i^{*}} \frac{1166\log T}{\Delta_{i}}
        +  4\log(K)\left(
        \max_{i\ne i^{*}}d_{i}(1-\Delta_{i}/4)
        + d_{i^{*}}(1-\min_{i\ne i^{*}}\Delta_{i}/4) \right)
    \end{align*}

\subsection{Proof of \Cref{thm:reward_depndent_lower}}
\label{appendix:dep-lower}
Consider two instances $I_1, I_2$. The rewards on both instances are sampled from Bernoulli distributions with $\mu_1 = \frac{1}{2}$, and the delay for arm $1$ is $\tilde{d}$ with probability $1-2\Delta$, and $0$ otherwise (regardless of the value of the reward).\footnote{The proof would hold even if there is no delay on arm $1$. However, arm $1$ would be sub-optimal under $I_2$, so we would like that $d_1(1-\Delta) = \tilde{d}$, as required by the theorem's statement.} Under $I_1$, arm $2$ is sub-optimal with $\mu_2 = \frac{1}{2} - \Delta$. The delay distribution for arm $2$ is as follows:
\begin{align*}
    \Pr\left( d_{t} = d ~\mid~a_t=2, r_{t} = 0, I_{1} \right)
    & =
    \begin{cases}
        \frac{4\Delta}{1+2\Delta}    & d = \tilde{d}\\
        \frac{1-2\Delta}{1+2\Delta}  & d = 0,
    \end{cases}\\
    \Pr\left( d_{t} = 0 ~\mid~a_t=2, r_{t} = 1, I_{1} \right)
    & = 1.
\end{align*}
Under $I_2$, arm $2$ is optimal with $\mu_i = \frac{1}{2} + \Delta$. The delay distribution for arm $2$ is:
\begin{align*}
    \Pr\left( d_{t} = 0 ~\mid~a_t=2, r_{t} = 0, I_{2} \right)
    & = 1,\\
    \Pr\left( d_{t} = d ~\mid~a_t=2, r_{t} = 1, I_{2} \right)
    & =
    \begin{cases}
        \frac{4\Delta}{1+2\Delta} & d = \tilde{d}\\
        \frac{1-2\Delta}{1+2\Delta} & d = 0.
    \end{cases}
\end{align*}
Note that under both instances, $d_2( 1 - 2\Delta) = \tilde{d}$. Since, if $a_t = 2$, then,
\begin{align*}
   & \Pr\left(d_{t}=0 \mid a_t = 2, I_{1}\right)  =\Pr\left(d_t=0 \mid a_t = 2, I_{2}\right)=1-2\Delta,\\
   & \Pr\left(d_{t} = \tilde{d} \mid a_t = 2, I_{1}\right)  =\Pr\left(d_{t} = \tilde{d} \mid a_t = 2, I_{2}\right)=2\Delta.
\end{align*}
Also note that the probability to \textit{observe} $0$ or $1$ given that the delay is $0$, is identical under both instances. That is,
\begin{align*}
    \Pr\left(r_{t}=1  \mid a_t=2, d_{t}=0,I_{1}\right) 
    & =
    \frac{\Pr\left(d_{t}=0\mid a_t=2, r_{t}=1,I_{1}\right)\Pr\left(r_{t}=1\mid a_t=2, I_{1}\right)}{\Pr\left(d_{t}=0\mid a_t=2,I_{1}\right)}
    \\
    & =
    \frac{\frac{1-2\Delta}{2}}{1-2\Delta}
    =\frac{1}{2},
\end{align*}
and, 
\begin{align*}
    \Pr\left(r_{t}=1\mid a_t=2, d_{t}=0,I_{2}\right) 
    & =
    \frac{\Pr\left(d_{t}=0\mid a_t=2, r_{t}=1,I_{2}\right)\Pr\left(r_{t}=1\mid a_t=2,I_{2}\right)}{\Pr\left(d_{t}=0\mid a_t=2,I_{2}\right)}\\
    & =
    \frac{\frac{1-2\Delta}{1+2\Delta}\frac{1+2\Delta}{2}}{1-2\Delta}
    =\frac{1}{2}.
\end{align*}
Until time $\tilde{d}$, the learner only observes feedback that has delay of $0$ which distribute the same on both instances. Furthermore, the amount of delayed feedback also behave identically since both arms in both instances have the same probability ($2\Delta$) to be delayed. Therefore, a learner cannot distinguish between $I_1$ and $I_2$, and in expectation she pulls the sub-optimal arm at least $\tilde{d}/2$ under one of the instances. 
% Hence, the in expectation. Formally,
% \begin{align*}
%     \mathcal{R}_T &\geq \max\big\{\mathcal{R}_T(I_1), \mathcal{R}_T(I_2) \big\} 
%     \\
%     &\geq \max\big\{\mathcal{R}_{d}(I_1), \mathcal{R}_{d}(I_2) \big\} \geq \frac{1}{2}d \Delta
% \end{align*}
Finally, assuming the learner guarantees a regret bound of $T^\alpha$ over any instance, using \cref{lemma:no-delay-lower-bound}, she suffers regret of
\begin{align*}
   \mathcal{R}_T 
   \geq
   \max\Big\{\frac{1}{2}d \Delta, \frac{(1-\alpha)\log(T)}{64\Delta}\Big\}
   \geq
   \frac{1}{4}d + \frac{(1-\alpha)\log(T)}{128\Delta}.
\end{align*}

\newpage
\section{Additional Experiments}
\label{sec:additional-exp}

\paragraph{More on $\alpha$-Pareto delays.} In \cref{fig:all_SEvsPB} we extend \cref{fig:SEvsPB} and reproduce the experiment done by \citet{manegueu2020stochastic} for other values of $\alpha_2$ that they tested: $\alpha_2 \in \{0.2,0.3,0.4,0.5,0.8\}$. All other parameters remains the same as in \cref{fig:SEvsPB}. That is, $T = 3000$ rounds, $K=2$ arms, $\alpha_1 = 1$, and the expected rewards are $\mu_1 = 0.4$ and $\mu_2 = 0.4 + \Delta$, where $\Delta \in [0.04, 0.6]$.

\begin{figure}[h]
    %\vskip 0.2in
    \centering
    \includegraphics[width = 1.0\textwidth]{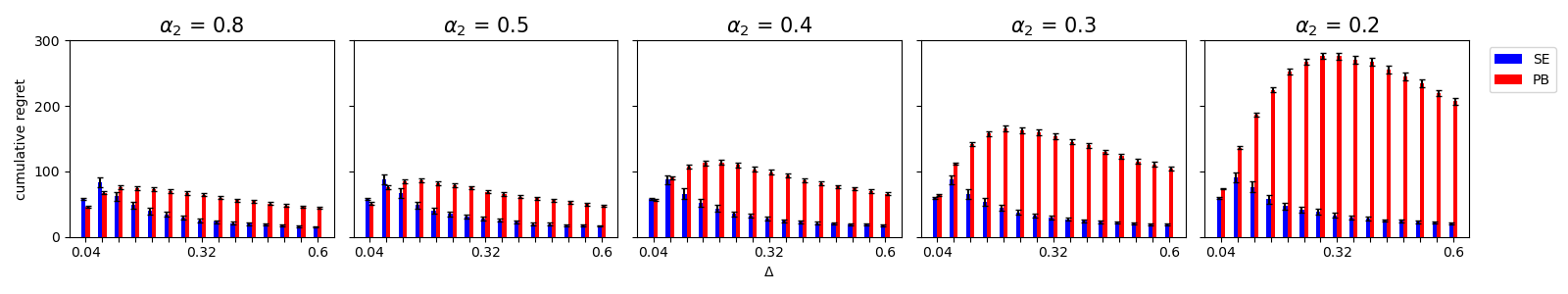}
    \caption{\label{fig:all_SEvsPB} Regret of SE and PatientBandits (PB) for Pareto delays.}  
    %\vskip -0.1in
\end{figure}

Recall that smaller values means heavier tail on the delay distribution (i.e. larger delays). For large values of $\alpha_2$, the performance of SE and PB is somewhat similar, with a minor advantage for PB under small values of $\Delta$ ($\Delta = 0.04,0.08$), and an advantage for SE under the rest of $\Delta$ values. Similar to the permanence UCB in \cref{fig:SEvsUSB} under fixed delays, as delay increases (that is, $\alpha_2$ decreases), the regret of PB increases as well. SE on the other hand, is almost not affected by the delay, as shown also for fixed delays in \cref{fig:SEvsUSB}. Around $\alpha_2 = 0.3$ SE becomes strictly superior for all values of $\Delta$. 
Naturally, PB is affected by $\alpha_2$ through the feedback, but is also affected directly as its confidence radius increases when the minimal $\alpha$ decreases. This also explains the fact that the peak of the PB curve moves right as $\alpha_2$ decreases, as more feedback is required to shrink the confidence interval enough so that the sub-optimal arm will be identified as such.

\end{document}